\documentclass[conference]{IEEEtran}

\usepackage[natbib=true,sortcites=no,style=numeric]{biblatex}
\bibliography{ref}

\ifCLASSINFOpdf
\else
\fi
\usepackage{amsmath}
\usepackage{amsthm}
\usepackage[utf8]{inputenc}
\usepackage{booktabs}
\usepackage{tabularx}
\usepackage[utf8]{inputenc} % allow utf-8 input
\usepackage[T1]{fontenc}    % use 8-bit T1 fonts
\usepackage{hyperref}       % hyperlinks
\usepackage{url}            % simple URL typesetting
\usepackage{booktabs}       % professional-quality tables
\usepackage{amsfonts}       % blackboard math symbols
\usepackage{nicefrac}       % compact symbols for 1/2, etc.
\usepackage{microtype}      % microtypography
\usepackage{nicefrac}
\usepackage{xspace}
\usepackage{times}
\usepackage{color}
\usepackage{times}
\usepackage{xfrac}

\usepackage{paralist} % needed for asparaitem for both normal article class and the NIPS specific template
\usepackage{psfrag}
\usepackage{dsfont}
\usepackage{paralist}
\usepackage{changepage}
\usepackage{framed,verbatim}

 \usepackage{enumitem}
 \setlist[description]{style=nextline}

\PassOptionsToPackage{hyphens}{url}
\usepackage{algorithm}
\usepackage[noend]{algpseudocode}
\usepackage{algpseudocode}

\algblock{Input}{EndInput}
\algnotext{EndInput}
\algblock{Output}{EndOutput}
\algnotext{EndOutput}
\newcommand{\Desc}[2]{\State \makebox[4em][l]{#1}#2}

\newcommand{\PIWP}{\mathsf{PIWP}}

\newcommand{\Bayes}{\mathsf{Bayes}}

\newif\ifcomment
\commentfalse
\ifcomment
\newcommand{\melissa}[1]{\textcolor{magenta}{\textbf{Melissa: #1}}}
\newcommand{\esha}[1]{\textcolor{cyan}{\textbf{Esha: #1}}}

\newcommand{\Snote}[1]{\textcolor{red}{\textbf{Saeed:#1}}}
\else
\newcommand{\melissa}[1]{\textcolor{magenta}{}}
\newcommand{\esha}[1]{\textcolor{cyan}{}}
\newcommand{\Snote}[1]{\textcolor{red}{}}
\fi

\usepackage{centernot}

\newcommand{\tD}{\tilde{D}}

\newcommand{\tX}{\tilde{X}}
\newcommand{\tY}{\tilde{Y}}

\newcommand{\Risk}{\mathsf{Risk}}
\newcommand{\crt}{\mathsf{crt}}

\newtheorem{theorem}{Theorem}

\newcommand{\fhat}[2]{\ifthenelse{\equal{#2}{}}{\hat{f}[#1]}{\ifthenelse{\equal{#2}{0}}{\hat{f}[\emptyset]}{\hat{f}[#1_{\leq #2}]}}}
\newcommand{\gain}[2]{\ifthenelse{\equal{#2}{}}{g[#1]}{g[#1_{\leq #2}]}}

\newcommand{\pr}[2][]{\Pr_{\ifthenelse{\isempty{#1}}{}{{#1}}}\left[{#2}\right]}

\makeatother

 \newcommand{\ANDT}{\wedge}

\usepackage{graphicx,accents}

\newcommand{\remove}[1]{}

\newcommand{\set}[1]{\{ #1 \}}

\newcommand{\R}{{\mathbb R}}
\newcommand{\N}{{\mathbb N}}

\newcommand{\cD}{{\mathcal D}}

\newcommand{\cH}{{\mathcal H}}

\newcommand{\cT}{{\mathcal T}}

\newcommand{\cX}{{\mathcal X}}
\newcommand{\cY}{{\mathcal Y}}

\usepackage{bm}
\newcommand{\eps}{\varepsilon}
\newcommand{\Exp}{\operatorname*{\mathbf{E}}}
\newcommand{\Ex}{\Exp}

\newcommand{\Supp}{\mathrm{Supp}}

\newcommand{\argmax}{\operatorname*{argmax}}

\newtheorem{claim}[theorem]{Claim}
\newtheorem{lemma}[theorem]{Lemma}
\newtheorem{corollary}[theorem]{Corollary}

\newtheorem{definition}[theorem]{Definition}
\newtheorem{example}[theorem]{Example}
\newtheorem{remark}[theorem]{Remark}

\newcommand{\sdotfill}{\textcolor[rgb]{0.8,0.8,0.8}{\dotfill}} %change to \cdotfill later

\newtheorem{proto}[theorem]{Protocol}
\newtheorem{protoc}[theorem]{Protocol}

\newcommand{\namedref}[2]{#1~\ref{#2}}

\newcommand{\torestate}[3]{%
\expandafter \def \csname BBRESTATE #2 \endcsname{#3}
\theoremstyle{plain}
\newtheorem{BBRESTATETHMNUM#2}[theorem]{#1}
\begin{BBRESTATETHMNUM#2}\label{#2}\csname BBRESTATE #2 \endcsname   \end{BBRESTATETHMNUM#2}
\newtheorem*{BBRESTATETHMNONNUM#2}{\namedref{#1}{#2}}
}

\newcommand{\restate}[1]{\begin{BBRESTATETHMNONNUM#1}[Restated] \csname BBRESTATE #1 \endcsname
\end{BBRESTATETHMNONNUM#1}}

\usepackage{xspace}

\newcounter{definitioncnt}
\newcounter{thmcnt}
\newcounter{prblm}
\newcommand{\myparagraph}[1]{\smallskip\noindent\textbf{#1}}

\newcommand{\poison}{\mathsf{poison}}

\begin{document}
%
% paper title
% Titles are generally capitalized except for words such as a, an, and, as,
% at, but, by, for, in, nor, of, on, or, the, to and up, which are usually
% not capitalized unless they are the first or last word of the title.
% Linebreaks \\ can be used within to get better formatting as desired.
% Do not put math or special symbols in the title.
\title{Property Inference from Poisoning}

% author names and affiliations
% use a multiple column layout for up to three different
% affiliations
 \author{\IEEEauthorblockN{Melissa Chase}
 \IEEEauthorblockA{
 Microsoft Research\\
 melissac@microsoft.com}
 \and
 \IEEEauthorblockN{Esha Ghosh}
 \IEEEauthorblockA{Microsoft Research\\
 esha.ghosh@microsoft.com}
 \and
 \IEEEauthorblockN{Saeed Mahloujifar}
 \IEEEauthorblockA{Princeton University\\
 sfar@princeton.edu}}

% conference papers do not typically use \thanks and this command
% is locked out in conference mode. If really needed, such as for
% the acknowledgment of grants, issue a \IEEEoverridecommandlockouts
% after \documentclass

% for over three affiliations, or if they all won't fit within the width
% of the page, use this alternative format:
% 
%\author{\IEEEauthorblockN{Michael Shell\IEEEauthorrefmark{1},
%Homer Simpson\IEEEauthorrefmark{2},
%James Kirk\IEEEauthorrefmark{3}, 
%Montgomery Scott\IEEEauthorrefmark{3} and
%Eldon Tyrell\IEEEauthorrefmark{4}}
%\IEEEauthorblockA{\IEEEauthorrefmark{1}School of Electrical and Computer Engineering\\
%Georgia Institute of Technology,
%Atlanta, Georgia 30332--0250\\ Email: see http://www.michaelshell.org/contact.html}
%\IEEEauthorblockA{\IEEEauthorrefmark{2}Twentieth Century Fox, Springfield, USA\\
%Email: homer@thesimpsons.com}
%\IEEEauthorblockA{\IEEEauthorrefmark{3}Starfleet Academy, San Francisco, California 96678-2391\\
%Telephone: (800) 555--1212, Fax: (888) 555--1212}
%\IEEEauthorblockA{\IEEEauthorrefmark{4}Tyrell Inc., 123 Replicant Street, Los Angeles, California 90210--4321}}

% use for special paper notices
%\IEEEspecialpapernotice{(Invited Paper)}

% make the title area
\maketitle

% As a general rule, do not put math, special symbols or citations
% in the abstract
\begin{abstract}

A major concern in training and releasing machine learning models is to what extent the model contains sensitive information that the data holders do not want to reveal. Property inference attacks consider an adversary who has access to the trained model and tries to extract some global statistics of the training data. In this work, we study property inference in scenarios where the adversary can maliciously control part of the training data (poisoning data) with the goal of increasing the leakage.   

Previous work on poisoning attacks focused on trying to decrease the accuracy of models either on the whole population or on specific sub-populations or instances. Here, for the first time, we study poisoning attacks where the goal of the adversary is to increase the information leakage of the model. Our findings suggest that poisoning attacks can boost the information leakage significantly and should be considered as a stronger threat model in sensitive applications where some of the data sources may be malicious.

We first describe our \emph{property inference poisoning attack} that allows the adversary to learn the prevalence in the training data of any property it chooses: it chooses the property to attack, then submits input data according to a poisoned distribution, and finally uses \emph{black box queries} (label-only queries) on the trained model to determine the frequency of the chosen property. We theoretically prove that our attack can always succeed as long as the learning algorithm used has good generalization properties. 

We then verify effectiveness of our attack by experimentally evaluating it on two datasets: a Census dataset and the Enron email dataset. In the first case we show that classifiers that recognizes whether an individual has high income (Census data) also leak information about the race and gender ratios of the underlying dataset. In the second case, we show classifiers trained to  detect spam emails (Enron data) can also reveal the fraction of emails which show negative sentiment (according to a sentiment analysis algorithm); note that the sentiment is not a feature in the training dataset, but rather some feature that the adversary chooses and can be derived from the existing features (in this case the words). Finally, we add an additional feature to each dataset that is chosen at random, independent of the other features, and show that the classifiers can also be made to leak statistics about this feature; this shows that the attack can target features completely uncorrelated with the original training task. We were able to achieve above $90\%$ attack accuracy with $9-10\%$ poisoning in all of these experiments.

\end{abstract}

% no keywords

% For peer review papers, you can put extra information on the cover
% page as needed:
% \ifCLASSOPTIONpeerreview
% \begin{center} \bfseries EDICS Category: 3-BBND \end{center}
% \fi
%
% For peerreview papers, this IEEEtran command inserts a page break and
% creates the second title. It will be ignored for other modes.
%\IEEEpeerreviewmaketitle
\maketitle

\section{Introduction}

    % \myparagraph{Privacy Preserving ML.}
    Machine learning is revolutionizing nearly every discipline from healthcare to finance to manufacturing and marketing. However, one of the limiting factors in ML is availability of large quantities of quality data. 
 
    This has prompted calls for collaborative learning, where many parties combine datasets to train a joint model~\citep{ericsson,medium}.  However, much of this data involves either private data about individuals or confidential enterprise information. In order to use data without revealing sensitive information about the training data, the field of privacy preserving machine learning is developing fast with significant progress on technologies like differential privacy and cryptography based learning systems.

    \myparagraph{Inference Attacks:} Inference attacks consider an adversary who tries to infer sensitive information about the training set by inspecting the model that is trained on it. Inference attacks have come in two main flavors: membership inference~\cite{shokri2017membership} and property inference attacks~\cite{ateniese2013hacking}.
    
    In a membership inference attack, an adversary tries to infer if a special instance was present in the training set that was used to train a given model.  In property inference attacks, the adversary tries to infer some aggregate information about the whole training set. While there are some promising approaches for defending against membership inference attacks (e.g. differential privacy), there is no general defense mechanism known against property inference attacks and how to defend against them is still an open question. In this work, we focus on property inference attacks in collaborative learning scenarios and show that these attacks are more devastating than previously thought.
    
    Note that the property being inferred need not be an explicit feature in the training set, nor does it need to be obviously correlated with the training set labels.  For example, we will consider a property inference attack on a text based model (in particular a spam classifier), which attempts to learn the average sentiment (positive or negative) of the documents in the training dataset.

    \myparagraph{Poisoning Attacks}  In poisoning attacks, some part of training data (poisoning data) is carefully chosen by an adversary who wishes to make the trained model behave in his own interest. A considerable body of works \cite{biggio2012poisoning,steinhardt2017certified,shafahi2018poison,mahloujifar2019curse, suya2020model, bhagoji2019analyzing, bagdasaryan2020backdoor} have shown that poisoning attacks can significantly hurt accuracy of ML models. 
    
    \myparagraph{Poisoning Attacks Increasing Information Leakage}
    In this work we initiate the study of poisoning attacks that aim at increasing the information leakage in ML models. In particular, we ask the following question:
    
        {\center{\emph{Can adversaries boost the performance of property inference attacks by injecting specially crafted poison data in the training set?}}}
        
    This is a very relevant question whenever data is gathered from multiple sources, some of which may be adversarially controlled.  In particular, it is relevant in collaborative machine learning, where one party might contribute malicious data in order to learn some property about the rest of the training set. 
    
    To follow our above example, if a group of small companies pool their data to train a spam classifier, one company might contribute poisoned data in an attempt to learn about the average sentiment in the rest of the group. This could give that company a significant edge in understanding it's competitors' positions.
    
    %. In this work we initiate the study of poisoning attacks again privacy of ML. Data is usually gathered from many sources and adversary can collude with one of the sources to improve the performance of the attack. 

    We note that the above question could also be asked for membership inference attacks.  While that is an interesting direction, in this paper we only focus on property inference attacks. We show that poisoning can indeed increase the property leakage significantly.

\myparagraph{Attack Model:} In this paper we consider an attacker who is allowed to first submit a set of ``poisoned'' data of its choice, which will be combined with the victim dataset and used to train a model. Note that our adversary is oblivious to the ``clean'' training set sampled from other sources\footnote{Oblivious poisoning attacks are provably weeker than non-oblivious poisoning attacks in what they can achieve ~\cite{garg2020obliviousness}.}.  The attacker can then make a series of black box queries to the trained model. Note that by black box queries we mean that the adversary only gets to learn the predicted label for each query. (It does not, for example, get the confidence values that the model produces, or any internal information from the model.)  Finally, as in the property inference attacks of \cite{propinf18, ateniese2013hacking}, the attacker's goal is to infer whether the average of a particular property is above or below a particular threshold. This is a very natural model for poisoning in the distributed learning setting where data is gathered from many sources.

\myparagraph{Main Contributions:}
\begin{itemize}
    \item{Theoretical Results:} We first describe a theoretical attack that works as long as the training algorithm outputs a Bayes-optimal classifier. The high level idea of our attack is that the adversary adds poisoning points in a way that causes the behavior of the Bayes-optimal classifier to depend on the average of the target property. In particular, we show that poisoning would change the prediction of certain instances when the average of property is below the threshold. But when the average is higher than the threshold, then the poisoning does not change the prediction of those points. See Section \ref{sec:theoreticalattack} for details of our analysis.

    We formalize this intuition 
    by giving a concrete attack based on our theoretical analysis in this model and analyzing it's effectiveness. Our attack is agnostic to the architecture of the trained model since it is completely black box. Note that poisoning is a crucial aspect of our theoretical analysis and the information leakage does not necessarily exist if we do not allow the adversary to poison the data. 

    \item{Experimental Results:} Real training algorithms do not always output Bayes optimal classifiers, so there is a question about whether the above results hold in practice.  To explore this we implement our attack on several target properties on a Census dataset and on the Enron email dataset; in the first setting we consider a model trained to predict whether an individual with a given feature set has high income. In the second setting, we consider a spam classifier trained on the emails. The objective of the attacker in all these experiments is to distinguish whether or not the target property appears with high frequency in the dataset. Our attack can successfully distinguish various ranges of \emph{higher} vs. \emph{lower} frequencies, e.g. ($5\%$ from $10\%$), ($40\%$ from $60\%$), ($30\%$ from $70\%$). We experiment with three types of target properties: 
        \begin{enumerate}
            \item Properties that are present in the training dataset as a feature, e.g., Gender and Race in the Census dataset.
            \item Properties that are not present as a feature in the training data, but which may be derived from those features e.g., Negative sentiment in emails (as determined by sentiment analysis) in the Enron dataset.
            \item Properties that are uncorrelated with the rest of the training data or classification task: for this experiment, we added an independently chosen random binary feature to each data entry in both the Census and the Enron data.   
        \end{enumerate}
    Experimenting with these variety of target properties demonstrate the power of our attack.  In our attacks, we were able to achieve above $90\%$ attack accuracy with about $10\%$ poisoning in all of these experiments. We run our experiments for logistic regression and also for fully connected neural networks (See Section \ref{sec:expt}.). 
    The whitebox attack in~\cite{PIA18} used Gender and Race as their target properties on Census dataset, so these experiments provide a reference for comparison, and show that with moderate poisoning $5-10\%$  our attack significantly improves  the performance both in accuracy and running time, and of course in the access model which is fully black box.  Moreover, we implemented the attack of ~\cite{PIA18} on our random features (the third type listed above), and show that the attack accuracy is barely above $50\%$, while our attack with $5-10\%$ poisoning is able to achieve near perfect accuracy.

\end{itemize}

\subsection{Organization}

The rest of the paper is organized as follows. In Section~\ref{sec:related}, we discuss the related work. In Section~\ref{sec:threat} we describe our threat model and formalize it. In Section~\ref{sec:theoreticalattack} we describe the theoretical attack for Bayes Optimal Classifiers, and in Section~\ref{sec:concreteattack} we show our concrete attack. We discuss the experiments in Section~\ref{sec:expt} and conclude with some future directions to explore in Section~\ref{sec:conclusion}.

\section{Related Work} \label{sec:related}
It is quite well known by now that understanding what ML models actually memorize from their training data is not trivial. As discussed above, there is a rich line of work that tries to investigate privacy leakage from ML models  under different threat models. Here we provide some more detail on several of the works which seem most related to ours.  %Among the various threat models and leakage patterns studied in the literature so far, the ones closest to ours is property inference leakage~\cite{ateniese2013hacking,PIA18}, so we will only focus on those here. 
For a comprehensive survey on the other privacy attacks on neural networks, please see~\cite{he2019privacy}. 

The global property inference attacks of~\cite{PIA15,PIA18} are the most relevant to us: here the adversary's goal is to infer sensitive global properties of the training dataset from the trained model 
that the model producer did not intend to share. We have already described some examples above.  Property inference attacks were first formulated and studied in~\cite{PIA15}. However, this initial approach did not scale well to deep neural networks, so ~\cite{PIA18} proposed a modified attack that is more efficient. The main differences from our attack are in the threat model: 1) our adversary can poison a portion of the training data and 2) in \cite{PIA15,PIA18} the adversary has whitebox access to the model meaning that it is given all of the weights in the neural net, while our adversary has only blackbox access to the trained model as described above. We experimentally compare our attack performance and accuracy with that of~\cite{PIA18} in Section~\ref{sec:expt}.

Another closely related attack is the more recent subpopulation attack~\cite{SubpopulationDP}. Here the adversary's goal is to poison part of the training data in such a way that only the predictions on inputs coming from a certain subpopulation in the data are impacted. To achieve this, the authors poison the data based on a filter function that specifies the target subpopulation. However, the goal of these subpopulation attacks is  to attack the accuracy of the model.
% rather than to learn any global property, a natural question is whether we can enhance their poisoning strategy with our attack to learn some global property of the target subpopulation. We leave this as future work.
% An example filter function that the authors suggest is the following: in a dataset with race and gender features, an adversary may want to harm the performance specifically for black men. So it will choose data from the underlying distribution where ``race'' $=$ ``black'' and ``gender'' $=$ ``male''. To poison the training data, the adversary picks some sample data that satisfy the filter function  and adds this to the training set with flipped labels. The hope is that if the filter function represents a good enough separation, then the learning algorithm will be able to learn the poisoned function for the targeted subpopulation. 
% While the goal of these subpopulation attacks is  to attack the accuracy of the model, rather than to learn any global property, a natural question is whether we can enhance their poisoning strategy with our attack to learn some global property of the target subpopulation. We leave this as future work.

In~\cite{melis19} the authors studied property leakage in the federated learning framework.  In federated learning the process proceeds through multiple rounds.  In each round each of $n>2$ parties takes the intermediate model and uses their own data to locally compute an update.  These updates are all collected by a central party and used to form the next intermediate model.  The threat model in \cite{melis19} is the following: $n$ parties participate in a ML training using federated learning where one of the participant is the adversary. The adversary uses the model updates revealed in each round of the federated training and tries to infer properties of the training data that are true of a subpopulation but not of the population as a whole. We note that in this threat model, the adversary gets to see more information than on our model, so this result is not directly comparable to ours.

We also note that our work is similar to~\cite{optimalGB} in spirit, where the authors develop Bayes optimal strategies for a \emph{grey box} membership inference attack\footnote{We note that, even though the authors claimed that there results hold for black-box attack, in reality, they addressed grey-box attacks as the adversary in their model gets the confidence scores for different classes and not just the prediction.}. The attack the authors study is membership-inference (as opposed to property inference), so the goal and technique of the attack and the assumptions they make about the model are completely different from ours. 

Another related work is the recent work of \citep{jagielski2020auditing} that studies poisoning attacks as a tool to evaluate differential privacy and obtain lower bounds for differential privacy of existing algorithms. This poisoning strategy improved the previous lower bound on the state of the art differentially private algorithms \citep{jayaraman2019evaluating}.

We also have to mention some of different type of  attacks that utilize the information leakage of the learning algorithms. Membership inference attacks try to infer whether a specific example has been in the training set or not are distinct from us in that the goal of attacker is different but are similar to our setting in assumptions behind the attack \cite{shokri2017membership,choo2020label,nasr2018machine,song2019auditing}. Model inversion attacks try to reconstruct part of the training set using access to the model which is a very strong attack model but only works on models that memorize too much about data  \cite{fredrikson2015model}. Backdoor attack that try to inject backdoors  through poisoning the training set and trigger the backdoors could in the test time. These attacks are similar to us as both attacks leverage poisoning \cite{chen17,goldblum2020data}. Another important attack models against machine learning is the study of adversarial examples where an attacker tries to add small perturbations during inference and cause the model to misclassify/mispredict the input\cite{goodfellow2014explaining,madry2017towards,ilyas2019adversarial}. Interestingly, the implications of these attacks and defenses proposed against them on privacy and membership inference is studied in \cite{song2019privacy}.

Finally, we want to mention the line of work on poisoning attacks with provable guarantees. There is a line of work on poisoning attacks against \emph{any} classifier  where the goal of adversary is to increase the probability of an undesirable property over the final model~\cite{mahloujifar2017blockwise,mahloujifar2018learning,mahloujifar2019can,diochnos2019lower,mahloujifar2019data,mahloujifar2020learning}. Since these attacks are very general, the adversary can set this undesirable property to be about the leakage of final model. However, these attacks require the adversary to know the ``clean'' dataset before injecting/replacing the poisoning examples. Moreover, they require the learning algorithm to be deterministic and also have some original probability of the bad property happening even when data is provided in a benign fashion. Our attack is distinct from those attacks as it is oblivious to the clean dataset. Also, our attack only requires the learning algorithm to have some generalization properties and does not need it to be deterministic or having some original chance of outputting a leaky model. Another line of work tries to show poisoning attacks with provable guarantees against learning algorithms that use convex loss minimization \cite{steinhardt2017certified,koh2018stronger,suya2020model}. Our attack is distinct from these results in that our attacks have provable gaurantees against \emph{any} learning algorithm that has good generalization and also it does not need the knowledge of clean training set. On the positive side, the work of \citep{diakonikolas2019robust} showed the power of robust statistics in overcoming poisoned data in polynomial-time. These works led to an active line of work \cite{diakonikolas2019sever,prasad2018robust,lai2016agnostic} exploring the possibility of robust statistics over poisoned data with algorithmic guarantees. These results do not apply to our settings as they either are about robust parameter estimation of distributions or about accurate classification in presence of outliers.

\section{Property Inference Attacks}
    There has been a series of work looking at to what extent a model leaks information about a certain \emph{individual record} in the training set, including work on using differential privacy~\cite{DMNS06} to define what it means for a training algorithm to preserve privacy of these individuals and technically how that can be achieved. However, leaking information on individuals is not the only concern in this context.  In many cases even the aggregate information is sensitive.

This type of aggregate leakage inspires a line of work started in~\cite{PIA15,PIA18} that looks at \emph{property inference} attacks, in which the attacker is trying to learn aggregate information about a dataset.  In particular, we focus here, as did~\cite{PIA18}, on an attacker who is trying to determine the frequency of a particular property in the dataset used to train the model.  Notice that this type of aggregate leakage is a global property of the training dataset and is not mitigated by adding differential-privacy mechanisms.

\myparagraph{Does property inference pose an important threat model?} Property inference attacks could reveal very sensitive in formation about the dataset. To illustrate the importance of the attack model, we provide some examples of such sensitive information that could be revealed. For further discussion on the importance of these attacks we refer the reader to previous work of \cite{propinf18} and \cite{ateniese2013hacking}.

\begin{example}
    Imagine a company wants to use its internal emails to train a spam classifier. Such a model would be expected to reveal which combination of words commonly indicate spam, and companies might be comfortable sharing this information.  However, using a property inference attack, the resulting model could also leak information about the aggregate sentiment of emails in the company, that could potentially be very sensitive. For example, if the sentiment of emails in the company turn negative near the financial quarter, it could mean that the company is performing below expectations.  
\end{example}
\begin{example}
    Similarly, a financial company might be willing to share a model to detect fraud, but might not be willing to reveal the volume of various types of transactions. 
\end{example}

\begin{example}
    Or a number of smaller companies might be willing to share a model to help target customers for price reductions etc, however such companies might not be willing to share specific sales numbers for different types of products.
\end{example}
%We feel that this presents both a valuable target for attackers as described in the above examples, and a good starting place for this line of work.\esha{did not understand this sentence}

       \myparagraph{Property leakage from poisoned datasets}
    One of the questions that is not addressed in previous work on property inference attacks is scenarios where adversary can contribute to the part of the training set.  This could occur either because one of the parties in  collaborative training behaves adversarially, or because an adversary can influence some of the input sources from which training data is collected (e.g. by injecting malware on some data collection devices).   Specifically, the adversary can try to craft special poisoning data so that it can infer the specific property that it has in mind. 
      Note that this is not prevented by any of the cryptographic or hardware assisted solutions: in all of these there is no practical way to guarantee that the data that is entered is actually correct.   
    
    This type of \emph{poisoning} attack has been extensively studied in the context of security of ML models, i.e., where the goal of the attacker is to train the model to miss-classify certain datapoints~\cite{biggio2012poisoning,shafahi2018poison,mahloujifar2019curse}, but to the best of our knowledge ours is the first work that looks at poisoning attacks that aim to compromise privacy of the training data.
    %\melissa{is this true?}\esha{think so :)}
    %where the attacker aims to modify the training data to obtain a model that is in some way suboptimal, 
    
    \myparagraph{Black box or white box model access} The information leakage of machine learning models could be studied in both white-box and black-box setting.
    In this paper, we consider the \emph{black box} model, where the attacker is only allowed to make a limited number of queries to the trained model. We show that these attacks can be very successful.  "Black box" attacks is sometimes used to refer to attacks which also have access to model's confidence values on each query~\cite{optimalGB}.  We emphasize here that we use the stricter notion of black box and our attacker will use only the model predictions. This type of attack is studied independently in \cite{choo2020label} where they study ``label-only'' membership inference attacks.  \melissa{add reference to label-only term?}\Snote{Done.}

% \myparagraph{Theory and experimentation}  The research in adversarial machine learning has been something of an arms race in recent years, with proposed heuristic defenses followed by new attacks and in turn new defenses.  Here we aim to take a more principled approach, where we first present and analyse a theoretical attack which we can show succeeds against any Bayes optimal classifier.  Then, we verify this result experimentally by implementing it and testing it against models with different architectures, and show that it does indeed succeed with very high probability.  The experiments show that these attacks are real, whereas the theoretical analysis gives intuition for the attacks and hopefully will help to offer some suggestion for where we might look in the future to either strengthen the attack or think about mitigation's.

\section{Threat model}
\label{sec:threat}
Before going through the threat model, we introduce some useful notation.

\myparagraph{Notation.} We use calligraphic letter (e.g $\cT$) to denote sets and capital letters (e.g. $D$) to denote distributions. We use $(X,Y)$ to denote the joint distribution of two random variables (e.g. the distribution of labeled instances). To indicate the equivalence of two distributions we use $D_{1}\equiv D_2$. By $x\gets X$ we denote sampling $x$ from $X$ and by $\Pr_{x\gets X}$ we denote the probability over sampling  $x$ from $X$.  We use $\Supp(X)$ to denote the support set of distribution $X$.  We use $p\cdot D_{1} + (1-p)\cdot D_2$ to denote the weighted mixture of $D_{1}$ and $D_2$ with weights $p$ and $(1-p)$.

    \myparagraph{Property Inference: } To analyze property inference, we follow the model introduced in \cite{propinf18}. Consider a learning algorithm $L\colon (\cX\times \cY)^* \to \cH$ that maps datasets in $\cT\in (\cX \times \cY)^*$ to a hypothesis class $\cH$. Also consider a Boolean property $f\colon \cX \to \set{0,1}$. We consider adversaries who aim at finding information about the statistics of the property $f$ over dataset $\cT \in (\cX \times \cY)^*$, that is used to train a hypothesis $h\in \cH$. In particular, the goal of the adversary is to learn information about $\hat{f}(\cT)$ which is the fraction of data entries in $\cT$ that has the property $f$ over data entries, namely $\hat{f} = \Ex_{(x,y)\gets \cT} [f(x)]$.\melissa{I think there were formatting issues for $\hat{f}$- check that the above eqn is now correct.}   More specifically the adversary tries to distinguish between $\hat{f}(\cT)= t_0$ or $\hat{f}(\cT)= t_1$ for some $t_0< t_1\in [0,1]$.%\melissa{$t_0<t_1$ seems more natural...}\Snote{Changed that, I will make sure that we are consistent with this everywhere else. It would be good if you can check too.} 
    We are interested in the black-box setting where the adversary can only query the trained model on several points to see the output label.  I.e. the adversary does not get the confidence values for his queries, or any information about the parameters of the model.\melissa{I added the prev sentence.  Just wanted to check though - is "parameters of the model" too strong?}\Snote{No it's fine.} This model is also known as label-only attacks and has been recently explored in the context of membership-inference attacks \cite{choo2020label}.

    \myparagraph{Formal Model:} To formalize this, we use distributions $D_{-}, D_{+}$ to denote the underlying distribution of the dataset for instances with $f(x)=0$ and $f(x)=1$ respectively.  Then we consider two distributions made by mixing $D_{-},D_{+}$ at different ratios, i.e., $$D_t\equiv t\cdot D_{+} + (1-t)\cdot D_{-}$$ 
    $D_t$ is the distribution where $t$ fraction of the points have $f(x)=1$. The adversary's goal is to distinguish between $D_{t_0}$ and $D_{t_1}$, for some $t_0<t_1$, by querying (in a black box way) a model $M$ that is trained on one of these distributions. %Note that $D_{t_1}$ and $D_{t_0}$ are obtained by mixing $D_{+}$ and $D_{-}$ in ratios $t_1$ and $t_0$ respectively.\melissa{removing this sentence b/c i think it's confusing and doesn't add much} %\melissa{this might be confusing since we don't actually find $t$.  Maybe just change it to saying the adversary's goal is to distinguish $D_{t_1}$ and $D_{t_0}$...}\Snote{Changed the sentence.}\esha{$D_{+}$ and $D_{-}$ were swapped, I changed it}. 
    In this attack, as in previous work \cite{PIA15,PIA18} we assume that the adversary can sample from $D_{-},D_{+}$. %(See Remark \ref{sampling} below.)
    %  , i.e. he knows how data is correlated and distributed except for the target property. See Remark \ref{sampling} below for discussion.

    \myparagraph{Property Inference with Poisoning:} We consider the poisoning model where adversary can contribute  $pn$ "poisoned" points to a $n$-entry dataset $T$ that is used to train the model. To the best of our knowledge, this is the first time that poisoning attacks against privacy of machine learning are modeled and studied.
    
    In order to measure the power of adversary in this model we define the following adversarial game between a challenger $C$ and an adversary $A$. Our game mimics the classic indistinguishability game style used in cryptographic literature.  As described above, $L$ is the learning algorithm, $n$ is the size of the training dataset of which  $p$ fraction are poisoned points selected by an adversary. $D_{-},D_{+}$ are the distributions of elements $x$ with $f(x)=0, 1$ respectively, and the goal of the attacker is to tell whether the fraction of points in the victim dataset that are from $D_{+}$ is   $t_0$ or $t_1$.%\melissa{$t_1,t_0$ are swapped from how they we defined them above. } \melissa{Also, this has $D_{-}$ twice.} \melissa{maybe we should not say "less than" or "more than", since the experiment only uses the exact fractions?}\Snote{I changed the sentence but let's discuss this. Our theoretical results work in the stronger setting, it might be good to add more experiments.} The distinguishing game bellow formalizes this adversarial model. 
    
    %\Snote{Maybe we can refer to hypothesis testing literature for justifying this way of defining the attack?}
    
    \myparagraph{Assumptions on the Adversary's Knowledge} The adversary has access to the following:
    
    \begin{description}
        \item [Sample access to conditional distributions] $D_+, D_-$ and $X_+$ (formally defined in Def.~\ref{def:conddist}). Note that the adversary has no knowledge about the distribution of the target feature. It is also worth noting that this assumption is the same as what is used in previous property and membership inference attacks ~\cite{shokri2017membership,PIA18,choo2020label,nasr2018machine,nasr2019comprehensive,song2020systematic}. \esha{TODO} In many cases this is a very reasonable assumption. For instance, in the Enron email example, it is not hard for an adversary to gather  sample emails with positive and negative sentiments. In other scenarios it might be the case that some subset of the data has been made public, or that the adversary has some data that follows a similar distribution. 
        \item [Blackbox access to the trained model] The adversary can\\ query the trained model to get output labels alone and nothing else.
        \item [Training algorithm and the features] This is completely reasonable since we are in the collaborative setting, where the adversary is one of the collaborators and each collaborator has the right to know how their data will be processed. It is also worth noting that most practical techniques for performing collaborative training in a decentralized and privacy-preserving way (e.g., secure multi-party computation, SGX) will require each collaborator to know the training algorithm and the features. Finally, it is never a good idea to try to achieve privacy by hiding the algorithm details as it is rather easy to obtain these information, e.g. see ~\cite{wang2018stealing}. %\esha{Saeed, please add citations}\Snote{Added the citation.}
    \end{description}
    
    % \begin{remark} \label{sampling} Our attacker is assumed to be able to sample from $D_{-}$ and $D_{+}$ and $X_+$. Note that we only need sampling access which is identical to the attacker having access to a dataset sampled from the distribution. This assumption is the same as what is used in previous property and membership inference attacks. In many cases we believe that that this is a very reasonable assumption. For instance, in the Enron email example, it is not hard for an adversary to gather a few sample emails with positive and negative sentiments.%\melissa{why do we mention spam and non-spam here? shouldn't it just be positive and negative mails?}\Snote{Right, changed the sentence.}
    %In other scenarios it might be the case that some small subset of the data has been made public, or that the adversary has some data that follows a similar distribution. \esha{cite some examples here, for e.g: "The adversary is also allowed access to an unlabeled $D_{aux}$ in https://arxiv.org/pdf/2004.00053.pdf}  % Therefore, it is better to assume that the adversary has oracle access to these distributions.  %Some of the previous work \cite{}\cite{} have already shown that  have shown that this assumptions could be relaxed in some applications with use of similar datasets for the case of membership inference attacks.  However, removing this assumption for property inference attacks is still an interesting open question.
    %\end{remark}
    
\begin{framed}

    \noindent$\PIWP(L, n, p, D_{-}, D_{+}, t_1, t_0)$:
    % \esha{mention that both parties have access to the parameters?}
\begin{enumerate}
        \item $C$ select a bit $b\in\set{0,1}$ uniformly at random. 
        % \esha{why does this need to happen after A sends poisoned data? This step can happen before the previous one, right?}\Snote{right, I changed the order}
         then samples a dataset of size $(1-p)\cdot n$: $\cT_{\mathsf{clean}} \gets D_{t_b}^{(1-p)n}$ 

        \item Given all the parameters in the game, $A$ selects a poisoning dataset $\cT_{\mathsf{poison}}$ of size $pn$ and sends it to $C$.
        \item $C$ then trains a model $M\gets L(\cT_{\mathsf{poison}}\cup \cT_{\mathsf{clean}})$.
        \item $A$ adaptively queries the model on a sequence of points $x_1,\dots,x_m$ and receives $y_1 =M(x_1),\dots, y_m=M(x_m)$.
        \item $A$ then outputs a bit $b'$ and wins the game if $b=b'$.
    \end{enumerate}
\end{framed}
    We aim to construct an adversary that succeed with probability significantly above 1/2.

\section{Attack against Bayes-Optimal Classifiers}\label{sec:theoreticalattack}
% \esha{notation using subscript 1,2 here vs 0,1 in previous sections.}
In this section, we will introduce a theoretical attack with provable guarantees for Bayes-optimal classifiers. A Bayes-optimal classifier for a distribution $D\equiv(X,Y)$ is defined to be a classifier that provides the best possible accuracy, given the uncertainty of $D$. Below, we first define the notion of Risk for a predictor and then define Bayes error and Bayes-optimal classifier based on that.

\begin{definition}[Risk of predictors]
For a distribution $D\equiv(X,Y)$ over $\cX\times \set{0,1}$ and a predictor $h\colon \cX \to \set{0,1}$, the risk of $h$ is defined as 
$\Risk(h,D)=\Pr_{(x,y)\gets D}[h(x)\neq y]$.
\end{definition}

    \begin{definition}[Bayes Error and Bayes-optimal classifier]\label{def:bayes} Let $D\equiv(X,Y)$ be a distribution over $\cX\times \set{0,1}$. The Bayes error of $D$ is defined to be the optimal risk for any deterministic predictor of $\set{0,1}$ from $\cX$. %\melissa{Do we need to define Bayes error formally? Do we use it below? Can we just say it's the error of the Bayes optimal classifier?}\Snote{We use it in Definition 7}
    Namely, $$\Bayes(D)=\inf_{h\colon \cX \to \set{0,1}} \Risk(h,D).$$

    Also the Bayes-optimal classifier for $D$ is defined to be a hypothesis $h^*_D$ that minimizes the error over the distribution. Such classifier always exists if the support of $Y$ is a finite set (which is the case here since $\Supp(Y)=\set{0,1}$.) %\melissa{$\set{0,1}$ is always a finite set....}\Snote{Right, this happened due to a find-replace operation:)}
    In particular, the following function is a Bayes-optimal classifier for any such $D$
    $$\forall x\in \cX: h^*_D(x)= \argmax_{y\in\set{0,1}} \Pr[Y=y\mid X=x].$$
    \end{definition}
    % \esha{As I understand, $\Bayes(D) = \{h_1, \ldots, h_t\}$ will give a class of hypothesis and $h*$ will be one of them if $\set{0,1}$ is finite. Is this correct? Shall we add that here just to clarify? Also, should this be $\argmax_{y\in\set{0,1}, (x,y) \gets D}$? }
    The Bayes error is the best error that a classifier can hope for. High performance learning algorithms try to achieve error rates close to Bayes error by mimicking the behavior of Bayes-optimal classifier. Below, we assume a learning algorithm that can learn the almost Bayes-optimal classifier for a class of distributions, and will show that even such a high quality learning algorithm is susceptible to attack. Let us first define the notion of approximate Bayes optimal learning algorithms.

    \begin{definition}[$(\eps,\delta)$-Bayes optimal learning algorithm]\label{def:fict} For two functions $\eps:\N\to [0,1]$ and $\delta:\N\to[0,1]$, a learning algorithm $L$ is called a $(\eps,\delta)$-Bayes optimal classifier for a distribution $D\equiv(X,Y)$ iff for all $n \in \N$, given a dataset $\cT\gets D^n$ with $n$ samples from $D$, the learning algorithm outputs a model $h$ such that
    $$\Risk(h,D) \leq \Bayes(D) + \eps(n).$$
    with probability at least $1-\delta(n)$ over the randomness of samples from the distribution. 
    \end{definition}
    % We know that the Bayes-optimal algorithm $L^*$ cannot exist because of the {\it No free lunch }\cite{wolpert1996lack} theorem.
    % \esha{add a citation? Also, I do not know this theorem, can we add a line informally describing what this theorem is?} 
    % However, for understanding the attack we imagine this learning algorithm exists and try to attack it. After that we will see that we can attack any learning algorithm that outputs some model ``close'' to the Bayes optimal classifier, for a class of (not all) distributions.  
Note that $\epsilon$ and $\delta$ are usually decreasing functions of $n$ that can converge to $0$. Now, we are ready to state our main theorem. But before that, we need to define the notion of certainty of a point which defines the (un)certainty of the response variable on a particular point $x$. We use this notion to identify the points that have a lot of ambiguity. We will see in our theorem below that if we have enough ambiguous points, we can run our attack.

\begin{definition}[Signed certainty] For a distribution $D\equiv(X,Y)$ over $\cX\times\set{0,1}$, the (signed) certainty of a point $x\in\cX$ according to $D$ is defined by
$$\crt(x,D) = 1-2\Pr[Y=1 \mid X=x].$$
\end{definition}    

\begin{definition}[Class of distributions induced by a property]\Snote{We need a better name here} Let $f:\Supp(X)\to\set{0,1}$ be a property and $D\equiv(X,Y)$ be a distribution of labeled instances. Let $X_+\equiv X\mid f(X)=1$ and $D_+\equiv (X,Y)\mid f(X)=1$ and $D_-\equiv (X,Y)\mid f(X)=0$. We use $\cD_f$ to denote the following class of distributions:
\begin{align*}\cD_f  = \set{\alpha_1\cdot (X_+, 1) + \alpha_2\cdot D_+ + \alpha_3\cdot D_-\\
~~~~~~~~~~~~~~~~~\mbox{ where }  \alpha_i\in [0,1], \sum_{i=1}^3\alpha_i=1}.
\end{align*}
\label{def:conddist}
\end{definition}%\esha{not being able to parse this} %\melissa{I'm confused by the formal equation for $D_f$ too.  What do the commas and semicolons mean?  Is what you mean?: $D_f = \alpha_1\cdot (X_+, 1) + \alpha_2\cdot D_+ + \alpha_3\cdot D_-$ where  $\alpha_i\in [0,1], \sum_{i=1}^3\alpha_i=1$}
%\Snote{Yes it is exactly this.}
%\melissa{I changed the equation to something that is hopefully easier to parse.}

We are now ready to present our main theorem which describes conditions where an almost Bayes optimal algorithm is vulnerable to attack.  We will prove this theorem in the following section.
    \begin{theorem}\label{thm:theoritical_attack}
Let $D\equiv(X,Y)$ be a distribution over $\cX\times \set{0,1}$ and $f\colon \cX \to \set{0,1}$ be a property over its instances. Let $D_+\equiv (X_+,Y_+) \equiv (X,Y) \mid f(X)=1$ and $D_-\equiv (X_-,Y_-)\equiv (X,Y) \mid f(X)=0$ be conditional distributions based on property $f$. Consider a learning algorithm $L$ that is $(\eps,\delta)$-Bayes optimal for class $\cD_f$. %
For any $p, t_0 < t_1 \in [0,1]$, 
if there exist $\tau\in[0,1]$

\begin{align*}
    \Pr_{x\gets X}\Big[&\frac{p+ 2\tau\cdot t_1}{t_1(1-p) } < \crt(x,D)\leq \frac{p-2\tau\cdot t_0}{t_0(1-p)}\\
    &\wedge~~~f(x)=1\Big]>\frac{2\epsilon(n)}{\tau}
\end{align*}

then there is an adversary $A$ who wins the security game\\ $\PIWP(n,L^*,D_{-},D_{+},p, t_0,t_1)$ with probability at least $1-2\delta(n)$. 
\end{theorem}

Theorem \ref{thm:theoritical_attack} states that our attack will be successful in distinguishing  $D_{t_0}$ from $D_{t_1}$ if there are enough points in the distribution with high uncertainty and the learning algorithm is Bayes-optimal for a large enough class of distributions.
\begin{remark}
One can instantiate Theorem~\ref{thm:theoritical_attack} by replacing $f$ with $1-f$. In this case, instead of $t_0$ and $t_1$ we need to work with $1-t_0$ and $1-t_1$. On the other hand, we can also flip the labels and use the distribution $(X,1-Y)$ instead of $(X,Y)$. Using these replacements, we can get four different variants of the theorem with different conditions for the success of the attack. In Section \ref{sec:concreteattack} we will see that in different scenarios we use different variants as that makes the attack more successful. 
\end{remark}

\begin{remark}[On the possibility of defending our attack with differential privacy]
Theorem \ref{thm:theoritical_attack} states that our attack will be successfully against \emph{any} learning algorithm that provides strong generalization bounds for both benign and poisoned distributions. 
  Differential privacy will not treat different distributions differently and we believe the only way it could be effective is by dropping the accuracy for both benign and poisoned distribution. Also, we have to note that differential privacy does not provide any guarantees in our setting as the information leaked through our attack is a global property of the dataset and not related to an individual example in the training set. 
%   In other words, there could exist a differential private algorithm that is susceptible to property inference attacks.
\end{remark}

    \subsection{Attack Description}
    In this section we prove Theorem \ref{thm:theoritical_attack}. We first describe an attack and then show how it proves Theorem \ref{thm:theoritical_attack}.
    
   The rough intuition behind the attack is the following. If an adversary can produce some poisoning data at the training phase to introduce correlation of the target property with the label, then this will change the distribution of training examples. This will change  the resulting classifier as well because the learning algorithm is almost Bayes optimal and should adapt to distribution changes. This change will cause the prediction of the uncertain cases (i.e., those which occur in the training set almost equally often with label 0 and 1) to change.  The adversary can choose the poisoning points in a way that the change of prediction is noticeably different for certain points between the cases when the target property occurs with frequency $t_0$ vs $t_1$. In the rest of section, we will show how adversary selects the poisoning points. Then we will see that the behavior of resulting model should be different on points that satisfy the conditions in Theorem \ref{thm:theoritical_attack}, depending on the distribution used during the training. At the end, we argue that by querying these points, the adversary can distinguish between the case of $t_0$ and $t_1$.

    Let the original data distribution of clean samples be $D\equiv{(X,Y)}$. Our adversary $A$ will pick the poison data by i.i.d. sampling from a distribution $D_A\equiv(X_A,Y_A)$. Note that this adversary is weak in a sense that it does not control the poison set but only controls the distribution from which the poison set is sampled. The resulting training dataset will be equivalent to one sampled from a distribution $\tilde{D}$ such that $\tilde{D}$ is a weighted mixture of $D$ and $D_A$. More precisely,
    \begin{equation*}
        \tilde{D} \equiv \left\{
                \begin{array}{lll}
                 D, & \text{With probability } (1-p) & \text{[Case I: No poisoning],}\\
                 D_A & \text{With probability } p & \text{[Case II: Poisoning]}
                \end{array}
              \right.
       \end{equation*}
       
  Now we describe the distribution $D_A$. To sample poisoning points, adversary first samples a point from $X$ conditioned on having the property $f$. For the label, the adversary always chooses label $1$. So we have $D_A \equiv (X_+, 1)$. We will see how this simple poisoning strategy can help the adversary to infer the property.
    
\subsection{Evaluating the attack}

In this section, we evaluate the effect of the poisoning strategy above on the Classifiers. We describe the evaluation steps here and defer proofs to the appendix.

\begin{description}
\setlength{\itemsep}{0pt}
\setlength{\parskip}{0pt}
\setlength{\parsep}{0pt}
 \item [Effect of poisoning on the distribution] Let $(\tX,\tY)$ be the joint distribution of samples from $\tD$. First we calculate $\Pr[\tY=1 \mid \tX=x]$ to see the effect of poisoning on the distribution. Let $E$ be the event that the point is selected by adversary, namely, the second case in the description of $\tD$ happens. Let $t=\Pr[f(X)=1]$, we prove the following claim.
    
    \begin{claim}\label{clm1}
        For any $x\in\cX$ such that $f(x)=1$ we have
        \begin{align*}
            \Pr[\tilde{Y}=1 \mid \tilde{X}=x] &= \frac{p}{p+t(1-p)}\\
            &~~~+ \frac{t(1-p)}{p+t(1-p)}\cdot \Pr[Y=1\mid X=x].
% &= 1 - \frac{t(1-p)}{p+t(1-p)}\cdot(\crt(x)+\frac{1}{2}).
        \end{align*}
    \end{claim}
    
 As a corollary to this claim, we prove that
 
 \begin{corollary}\label{cor1}
For any $x$ such that $f(x)=1$ and for any $\tau\in\R$ we have $\Pr[\tY=1\mid \tX=x]\geq \frac{1}{2} + \frac{\tau\cdot t}{p+t(1-p)}$ if and only if $\crt(x)\leq \frac{p-2\tau\cdot t}{t(1-p)}$ %\esha{ Add: "for any $\tau$"?}.
\end{corollary}
Claim \ref{clm1} and Corollary \ref{cor1} show how the adversary can change the distribution. We now want to see the effect of this change on the behavior of the Bayes optimal classifier. 
\item [Effect on the Bayes Optimal Classifier] Consider the algorithm $L$ that is $(\eps,\delta)$-Bayes optimal on all linear combinations of distributions $D_+$, $D_-$ and $D_A$ (as stated in Theorem~\ref{thm:theoritical_attack}). Therefore, on a dataset $\tilde{\cT}\gets \tilde{D}^n$, the algorithm $L$ will output a model $\tilde{h} = L(\tilde{T})$ that with probability at least $1-\delta(n)$  has error at most $\Bayes(\tilde{D})+\eps(n)$. Consider joint distribution $(\tX, \tY)$ such that $\tilde{D}\equiv (\tX,\tY)$. 
The following claim shows how the adversary can exploit the dependence of the probabilities on $t$ and infer between $t_0$ and $t_1$ by using special points that have high uncertainty.

\begin{claim}\label{clm2}
Let $L$ be a $(\epsilon,\delta)$-Bayes optimal learning algorithm for $\tD\equiv(\tX,\tY)$. Consider an event 
$C_\tau$ defined on all $x\in\cX$ such that $C_\tau(x)=1$ iff $f(x)=1$ and
\begin{align*}
\frac{p + 2\tau\cdot t_1}{t_1(1-p)}< \crt(x)\leq \frac{p- 2\tau \cdot t_0}{t_0(1-p)}  .
\end{align*}

If there exist a $\tau\in[0,1]$ and $\gamma\in[0,\frac{1}{2}]$  such that we have $\Pr[C_\tau(X)=1]\geq \frac{\eps(n)}{\tau\cdot(1-2\gamma)}$

then
if $t=\Pr[f(X)=1]= t_0$ we have
\begin{align*}
             \Pr_{\substack{S\gets \tilde{D}^n\\
             \tilde{h}\gets L(S)}}\left[\Pr_{x\gets X|C_\tau(x)=1}\left[\tilde{h}(x) =1\right] \geq 0.5 + \gamma\right]\geq 1-\delta(n)
   \end{align*}
  and if $t=\Pr[f(X)=1]= t_1$
  \begin{align*}
             \Pr_{\substack{S\gets \tilde{D}^n\\
             \tilde{h}\gets L(S)}}\left[\Pr_{x\gets X|C_\tau(x)=1}\left[\tilde{h}(x) =1\right] \leq 0.5 - \gamma\right]\geq 1-\delta(n).
   \end{align*}

\end{claim}

    \item [Putting it together] Using Claim \ref{clm2}, we will finish the proof of Theorem \ref{thm:theoritical_attack}. Intuitively, the adversary wants to calculate the probability $\Pr_{x\gets X|C_\tau(x)=1}\left[\tilde{h}(x) =1\right]$. For doing this we use standard sampling methods to estimate the probability. 
    
        \begin{itemize}
            \setlength{\itemsep}{0pt}
\setlength{\parskip}{0pt}
\setlength{\parsep}{0pt}
            \item The adversary first samples $m=\frac{-\log(\delta(n))}{2\gamma^2}$ samples from $(X,Y) \mid C_\tau(X)=1$. In order to do this, adversary needs to sample roughly $m\cdot \frac{\tau(1-2\gamma)}{\eps(n)}$ points from $X_+$. %\esha{should this be $X_+$?}\Snote{That's right.}
            \item the adversary then calls $\tilde{h}$ on all the points and calculates the average of the prediction as $\rho$.
        \end{itemize}
    In the case where $\Risk(h,\tilde{D})\leq \epsilon(u)$ and $t=t_0$, Using claim \ref{clm2} and Chernoff-Hoefding inequality we have

            $$\Pr[\rho < 0.5] \leq \delta(n).$$

    Similarly in the case where $t=t_1$, and $\Risk(h,\tilde{D})\leq \epsilon(n)$, we have

            $$\Pr[\rho > 0.5] \leq \delta(n).$$

Therefore, the adversary only checks if $\rho>0.5 $  and based on that decides if $t=t_0$ or $t=t_1$. The probability of this test failing is at most $2\delta(n)$ as there is at most $\delta(n)$ probability of the learning algorithm failing in producing a good classifier and $\delta(n)$ probability of failure because of the Chernoff-Hoefding. Therefore, with probability at least $(1-2\delta(n))$ the adversary can infer whether $t=t_0$ or $t=t_1$.   
\end{description}

%!tex root = main.tex
\section{A concrete attack}\label{sec:concreteattack}
Here we describe the concrete attack we use in our experiments.  We note that there are many possible variations on this attack; what we have presented here is just one configuration that demonstrates that the attack is feasible with high accuracy.  We leave exploration of some of the other variations to future work.

\myparagraph{Selecting Poisoning Points}  Recall that the attacker is assumed to be able to sample from $D_{-}$ and $D_{+}$, the distribution of items with $f(x)=0$ and with $f(x)=1$.

As in the theoretical attack described in Section \ref{sec:theoreticalattack}, the poisoned data is generated by sampling from $D_{+}\equiv (X_+,Y_+)$ and introducing correlation between the target feature $f(x)$ and the label by injecting poisoning points  of form $(X_+,1)$.

Note that the attack described in the Section \ref{sec:concreteattack} could be achieved in 4 different forms. Namely, the adversary can use any of the possible combinations $(X_-, 1)$, $(X_-,0)$, $(X_+,0)$ and $(X_-,1)$ for poison data. In algorithm \ref{alg:A3} we show how we choose between these strategies. In particular, when values of $t_0$ and $t_1$ are large, we try to attack $1-f$ instead of $f$. The logic behind this choice is that it is easier to impose a correlation between the property and the label, when the property is rare in the distribution. As an example, consider a spam detection scenario where the adversary wants to impose the following rule on the spam detection: all Emails that contain the word "security" must be considered spam. It would be much easier for the adversary to impose this rule when there are very only a few Emails containing this word in the clean dataset. In other words, if this rule is added to the spam detector, the accuracy of the spam detector on the rest of the emails would not change significantly as there are not many email containing the word "security".   On the other hand, we select the correlation in the direction that is not dominant in the data. Namely, if we have $Pr_{(x,y)\gets (X,Y)}[y=1|f(x)=1]\geq 0.5$ then we either use $(X_+, 0)$ or $(X_-,1)$ based on the values of $t_0$ and $t_1$. The intuition behind this choice is that we always select the correlation that happens less often in the training set so that the poisoning distribution is more distinct from the actual distribution and hence makes a larger change in the distribution. Note that we can just stick to one of these four options and still get successful attacks for most scenarios but our experiments shows that this is the most effective way of poisoning. Also, it is important to note that Theorem \ref{thm:theoritical_attack} could be extended to all of these attacks with slightly different conditions. The details of the poisoning algorithm are described in Algorithm \ref{alg:A1}.

\myparagraph{Selecting Query points:} The next challenge here is in choosing the query points.  As in Section \ref{sec:theoreticalattack}, we want to find query points whose certainty falls in a range close to 0. For instance, if the poisoning rate $p=0.1$ and we want to distinguish between $t_0=0.3$ and $t_1=0.7$, the Theorem suggests that we should be querying the points whose certainty falls between $[\approx 0.15, \approx 0.37]$. 
 In order to do this, we need to first calculate the certainty.  Since we only have sampling access to the distribution, we do not know the exact certainty.  Instead we approximate the certainty by estimating the probability that the label is $0$ or $1$ through training an ensemble of models and then evaluating all of them on a point. In particular, for a point $x$ we estimate  $\Pr[Y=1\mid X=x]$ using the fraction of models in the ensemble that predict $1$. Then we use this estimate to calculate the certainty. 
The way we estimate the certainty is obviously prone to error. To cope with the error, we actually work with a larger range than what is suggested by our theoretical result. 

In out attack, we fix the range of certainty to  $[-0.4,0.4]$ and query the points whose certainty falls in this interval. Although this range might be larger (and sometimes smaller) than what our Theorem suggests, but we still get good results in our attack. The reason behind this is that in the next step of the attack,  which is the inference phase, we filter the query points and only look at the important ones. So if there are query points that are not relevant, they will be ignored in the inference.

\myparagraph{Guessing the Fraction:} At the end, the adversary must  use the results of its  queries to come up with the prediction of whether $t=t_0$ or $t=t_1$. The theoretical attack suggests that we should just take the average of all responses and predict based on whether or not the average is less than $0.5$. However, as we pointed out before, we cannot use the exact range suggested by the theorem because we do not have access to exact certainty. To get around this issue, we train a linear model for the attack to identify the queries that are actually relevant in predicting whether $t=t_0$ or $t=t_1$. Here we use a shadow model approach: we sample multiple datasets $\cT^0_1,\dots,\cT^0_k$ and $\cT^1_1,\dots,\cT^1_k$ where each for each $i\in[k]$ we have $\Pr_{(x,y)\gets \cT^0_k}[f(x)=1]=t_0$ and $\Pr_{(x,y)\gets \cT^1_k}[f(x)=1]=t_1$. Now we use the poisoning dataset $\cT_\poison$ and the query dataset $\cT_q$ from the previous steps of the attacks as follows: We first use the poisoning set to poison all the datasets and train multiple models $(M^0_1,\dots,M^0_k)$ and $(M^1_1,\dots,M^1_k)$ where $M_0^i=L(\cT_0^i\cup \cT_p)$ and $M_0^i=L(\cT_0^i\cup \cT_p)$. After that we query each of the models on each of the chosen query points.  Finally, we generate a training set to train the attacker's model where the query responses are the features and the label are 0 or 1 depending on whether the "shadow" model used was from $\cT^0_i$ or $\cT^1_i$, and train a linear model $M_A$ on this set. The adversary then uses the query points to query the real target model, feeds the responses to $M_A$, and outputs whatever $M_A$ produces as it's final guess for whether $t=t_0$ or $t=t_1$.

Here we describe the concrete attack for distinguishing $D_{.3}$ from $D_{.7}$ with initial victim training set size $n$ and poisoning rate $p$ (i.e. $pn$ poisoned points out of $n$ total training points). Recall that the attacker has sampling access to $D_{-},D_{+}$, the distribution of samples with $f(x)=0$ and $f(x)=1$ respectively. The attack will proceed as follows:
\begin{algorithm}
\caption{Choosing poisoning data}
\small
\begin{algorithmic}
\Input
    \Desc{$f$}{The property being attacked}
  \Desc{$t_0,t_1$}{Possible fractions of instances with property $f$}
\Desc{$p$}{Poisoning ratio}
\Desc{$n$}{Size of training set}
\Desc{$D_+, D_-$}{Sampling oracle to distribution of instances with and without property $f$}
  \EndInput
  \Output
  \Desc{$\cT_\poison$}{Set of poisoning examples}
  \EndOutput
  \end{algorithmic}
\label{alg:A1}
\begin{algorithmic}[1]
    \If {$(t_0 + t_1)< 1$}
    \State{Sample $m=p\cdot{n}$ examples from $D_{+}$,\\ $$T=\set{(x_1,y_1),\dots, (x_{m},y_{m})}\gets D_+^{m}$$}
    \Else
    \State{Sample $m$ examples from $D_{-}$,\\ $$T=\set{(x_1,y_1),\dots, (x_{m},y_{m})}\gets D_-^{m}$$}
    \EndIf
    \State{$\alpha=\frac{\sum_{i=1}^{m}{y_i}}{m}$.}
    \If{$\alpha>0.5$}
    \State{$\cT_\poison = \set{(x_1,0),\dots,(x_{m},0)}$}
    \Else
    \State{$\cT_\poison = \set{(x_1,1),\dots,(x_{m},1)}$}
    \EndIf

    \State{Output $\cT_\poison$}
\end{algorithmic}
\end{algorithm}

\begin{algorithm}
\caption{Choosing the black box queries}
\label{alg:A2}
\small
\begin{algorithmic}
\Input
  \Desc{r}{number of models in ensemble}
    \Desc{q}{Number of black-box queries}
  \Desc{$D_+, D_-$}{Sampling oracle to distribution of instances with and without property $f$}
%   \Desc{$D_-$, $D_+$}{ Sampling oracles to distribution}

  \EndInput
  \Output
  \Desc{$\cT_q$}{Set of black-box queries}
  \EndOutput
  \end{algorithmic}
  \begin{algorithmic}[1]
    \State{Sample a thousand data sets $\cT_1,\dots, \cT_{r}$ each composed half of elements sampled from $D_{-}$ and half of elements sampled from $D_{+}$}. 
    \State{Train $M_1,\dots, M_{r}$ using $\cT_1,\dots,\cT_{r}$}.
    \State{Set $\cT_q=\emptyset$.}
    \While {$|\cT_q| < q$}
        \State{ sample $x\gets \frac{1}{2}D_- + \frac{1}{2}D_+$ }
        \State{if $$|1 - 2\frac{\sum_{i=1}^{r} M_i(x)}{r}|\leq 0.4$$
        $\cT_q = \cT_q \cup \set{x}.$}
    \EndWhile
    \State{Set $\cT_q=\cT_q\cup \cT_{\mathsf{poison}}$}
    \State{Output $\cT_q$ as the set of black box queries to make.}

\end{algorithmic}
\end{algorithm}
\small
\begin{algorithm}
\caption{Guessing the fraction of samples with property $f$}
\label{alg:A3}

\begin{algorithmic}
\Input
    \Desc{$f$}{The property being attacked}
  \Desc{$t_0,t_1$}{Possible fractions of instances with property $f$}
\Desc{$p$}{Poisoning ratio}
\Desc{$n$}{Size of training set}
\Desc{$k$}{Number of shadow models}
\Desc{$\cT_\poison$}{The set of poisoning data}
\Desc{$\cT_q$}{The set of black-box queries}
\Desc{$D_+, D_-$}{Sampling oracle to distribution of instances with and without property $f$}
  \EndInput
  \Output
  \Desc{$b'$}{A bit that predicts whether $t=t_0$ or $t=t_1$.}
  \EndOutput
  \end{algorithmic}
  \begin{algorithmic}[1]
    \State{ Sample data sets $\cT^{0}_1,\dots, \cT^{0}_{k}$ with size $n$ from $D_{t_0}$ and $\cT^{1}_1,\dots, \cT^{1}_{k}$ with size $n$ from $D_{t_1}$. (Note that the adversary can generate samples from these distributions given sampling access to $D_{-},D_{+}$: e.g. to sample from $D_{t_0}$, first choose bit $b$ from a distribution that is $1$ with probability $t_0$, then sample from $D_b$.)}
  
    \State{ Train $M^{0}_1,\dots, M^{0}_{k}$ using $\cT^{0}_1\cup T_\poison,\dots,\cT^{0}_{k}\cup T_\poison$ and $M^{1}_1,\dots, M^{1}_{k}$ using $\cT^{1}_1\cup T_\poison,\dots,\cT^{1}_{k}\cup \cT_\poison$.}
    \State{ Query all models on $T_q$ to get $R^{1}_1,\dots, R^{1}_{k}$ and $R^{0}_1,\dots, R^{0}_{k}$.}
    \State{ Construct a dataset $$\set{(R^{1}_1,1),\dots, (R^{1}_{k},1),(R^{0}_1,0),\dots, (R^{0}_{k},0)}$$ and train a linear model  with appropriate regularization on it to get $M_A$ (We use  $\ell_2$ regulizer with weight $2\cdot \sqrt{1/k}$)}

    \State{Query the target model on $\cT_q$ to get $R_q$, evaluate $M_A(R_q)$, and output the result.}
\end{algorithmic}
\end{algorithm}
\newcommand{\wb}{\mathsf{WBAttack}}
\section{Experimental Evaluation}\label{sec:expt}
Here we evaluate the performance and the accuracy of our attack described in Section~\ref{sec:concreteattack}.

\subsection{Experimental Setup}
\myparagraph{DataSets} We have run our experiments on two datasets:

\begin{itemize} 
\setlength{\itemsep}{0pt}
\setlength{\parskip}{0pt}
\setlength{\parsep}{0pt}
\item{\textbf{Census:}} The primary dataset that we use for our experiments is the US Census Income Dataset~\cite{frank2011uci}. The US Census Income Dataset contains census data extracted from the 1994 and 1995 population surveys conducted by the U.S. Census Bureau. This dataset includes 299,285
records with 41 demographic and employment related attributes
such as race, gender, education, occupation, marital status and
citizenship. The classification task is to predict whether a person
earns over \$50,000 a year based on the census attributes.\footnote{We used the census dataset as is, as compared to~\cite{PIA18, propinf18} 
where they preprocess the census dataset and run their experiments with balanced labels (50\% low income and 50\% high income). We notice that in the original dataset, the labels are not balanced (around 90\% low income and 10\% high income).
}
\item {\textbf{Enron:}} The second dataset that we use to validate our attack is the Enron email dataset~\cite{klimt2004enron}. This dataset contains 33717 emails. The classification task here was to classify an email as "Spam" or "Ham" based on the words used in the email.

\end{itemize}

\myparagraph{Target Property}  In Table~\ref{tab:targetfeatures}, we summarize the features we experimented with. In all these experiments, the attacker's objective is to distinguish between two possible values for the frequency of the target feature. Below is a summary list of all these properties.

\begin{table}[htbp]
\begin{center}
\scalebox{1}{
\begin{tabular}{|c|c|c|}
\hline
\textbf{DataSet} & \textbf{Target Feature} &\textbf{Distinguish between}\\
\hline
Census & Random binary & 0.05 vs 0.15 \\
\hline
Census & Gender & 0.6 vs 0.4 female\\
\hline
Census & Race & 0.1 vs 0.25 black\\
\hline
Enron & Random binary & 0.7 vs 0.3 \\
\hline
Enron & Negative sentiment & 0.10 vs 0.05 \\
\hline
\end{tabular}}
\caption{Target Features. For each feature we use two different ratios close to the actual ratios in the dataset (except for random features which are not present in the dataset) in our experiments. 
}
\label{tab:targetfeatures}
\end{center}
\end{table}

\begin{itemize}
\setlength{\itemsep}{0pt}
\setlength{\parskip}{0pt}
\setlength{\parsep}{0pt}
\item{\textbf{Random:}} To understand the power of this attack on a feature that is completely uncorrelated to the classification task (which one might naturally think  should not be leaked by an ideal model
), we did a set of experiments where we added a random binary feature to both Census and Enron datasets and set that as the target feature that the adversary wants to attack. Note that this feature is not correlated with any other feature in the dataset and the model should not depend on it to make its decision. This is backed up by our experiments in that, as we will see, the attack of \citep{PIA18}, which uses no poisoning 
does not perform better than random guessing on this property.

\item{\textbf{Gender:}} Gender is a boolean feature in census data which takes values "Male" and "Female". This feature was also used in the work of \citep{propinf18,PIA18}, so it allows us to directly compare our work with theirs.

\item{\textbf{Race:}} Another feature that we attack in the census dataset is Race. In our attack we tries to infer between two different ratios of "Black" race in the dataset. Again, we chose this for purposes of comparison with previous work.

\item{\textbf{Negative Sentiment:}} In one of our target properties, we try to infer the fraction of emails in the Enron email dataset that have negative sentiment. To do this, we use the sentiment analysis tool in {\it python nltk} to identify emails with positive and negative sentiment. Note that unlike all the other target properties, the negative sentiment 
feature is not present as a feature in the dataset which makes it intuitively seem harder for the attacker to infer. However, as we will see in our experiments, the attacker can still attack this property.
\end{itemize}

\subsection{Black-box queries}
As mentioned before, we are interested in the information leakage caused by black-box access to the model. Namely, the adversary can query the model on a number of points and infer information about the target property using the label prediction of the model on those queries (See Section \ref{sec:threat} for more details). In a concurrent work \citep{choo2020label}, also explored this kind of black-box access in the context of membership inference attacks. Our model does not require any other information other than predicted label (e.g. no confidence score, derivative of the model, etc). The query points are selected according to algorithm~\ref{alg:A2} in such a way that they have low certainty.

For Enron experiments, we use 500 query points and for census data experiments we use 1000 query points. The reason that we use different numbers for the Enron and census experiment is that Enron dataset contains fewer of ``uncertain'' points; we used almost all the points that fall into our range.

\subsection{Target model architectures}
Most of our experiments use logistic regression as the model architecture for training. The main reason we picked logistic regression was because it is much faster to train compare to Neural Networks. However, we also have a few experiments on the more complex models. In particular, we test our attack on fully connected neural networks with up to 6 hidden layers (See Table \ref{tab:modelcomplexity}). We note that since our attack is black-box, we do not need any assumption over the target model architecture other than the fact that it will have high accuracy (we still need the adversary to know the architecture for the attack to work.). This is again in contrast with the previous work of \citep{PIA18} 
that only works on fully connected neural networks. This is an important distinction as it shows that the attacker is not relying on extensive memorization due to large number of parameters.  

\subsection{Shadow model training}
Our shadow model training step is quite simple. As described in Section \ref{sec:concreteattack}, we train a series of shadow models with a fixed poisoning set. We hold out around $\frac{1}{3}$ of the dataset (For both Enron and Census datasets)  training the shadow models. This held out part will not effect any of the models that we test our attacker's accuracy on. After training the shadow models, we query them and train a simple linear attack model over the predictions on the queries. We use a linear model to train the attack model since our theoretical results suggest a simple linear model (which just takes the uniform average) over the queries would be enough to make a correct prediction.
We use $\ell_2$ regularization for our linear models to get better generalization and also reduce the number of effective queries as much as possible. Note that this choice of simple linear models is contrast with the attack of \citep{propinf18} and \citep{PIA18}.

that use complex models (e.g. set neural networks) to train the attack model; this is one of the reasons that our attack is faster.

\subsection{Performance of our attack}
In the following experiments, we evaluate the performance of our attack and compare it with the attack of~\cite{PIA18}. In the rest of the manuscript, we denote the attack of~\cite{PIA18} as $\wb$. \Snote{TODO: Add a a explanation of the $\wb$ attack}. We first evaluate how the different parameters, namely, poisoning rate, training set size, number of shadow models (defined in Sec~\ref{sec:threat}) and the complexity of the target model affect its accuracy. To understand the effect of each parameter, for each set of experiments, we fix a set of parameters and vary one. Unless otherwise stated, all of our experiments are repeated 5 times and the number reported is the average on all 5 repetitions.

\myparagraph{Poisoning Rate }
In Fig.~\ref{fig:logisticpoison}, we have 6 experiments where we fix the model to be logistic regression for all of them except one (Census random MLP) which uses a 5 layer perceptron with hidden layers sizes 32, 16 and 8. In all the experiments we set the number of shadow models to be 500 for census experiments and 1000 for Enron experiments. The training set size is 1000 for Census experiments and 2000 for Enron experiments.  We vary the poisoning rate from 0\% to 20\%. The number of black-box queries is set to 500 for experiments on Enron and 1000 for experiments on Census. The attack accuracy for all the target features is quite low when there is no poisoning. But with increase in poisoning rate, the attack accuracy dramatically improves and for  all features, the accuracy reaches around 0.9 with just 10\% poisoning.

\begin{figure}
\begin{center}
    \includegraphics[scale=0.6]{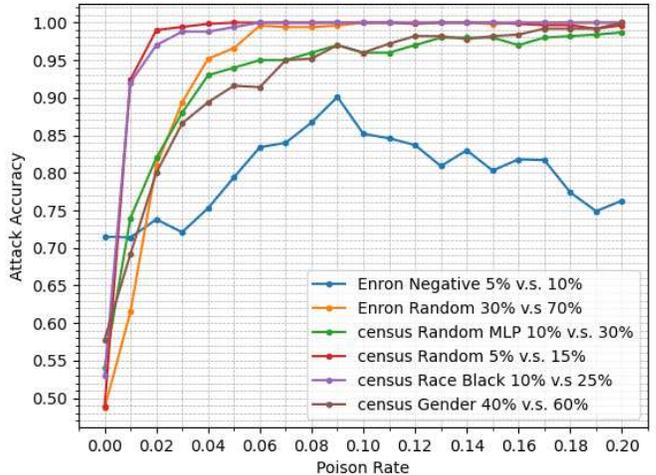}
    \caption{Poison rate vs attack accuracy. We change the poisoning budget from $0\%$ of the training size to $20\%$. Note that in most experiments, property inference  without poisoning is not effective, but with less than $10\%$ poisoning all of the experiments get accuracy above $90\%$. The curve marked as Census Random MLP shows the performance of our attack against MLP classifier on census dataset with random target feature. Also, observe that for one of the experiments (marked as Enron Negative), the accuracy of the attack starts decreasing after a certain level of poisoning, we discuss this phenomenon below.}
    \label{fig:logisticpoison}
\end{center}    
\end{figure}

The Enron negative sentiment experiment seems like an anomaly in Figure \ref{fig:logisticpoison}. However, the drop of accuracy with more poison points could be anticipated. We posit that for all features there is some point where adding more poisoning points will cause the accuracy to decrease.  To understand this, one can think about the extreme case where $100\%$ of the distribution is the poison data, which means there is no information about the clean distribution in the trained model and we would expect the attack to fail.

This effect is especially pronounced for properties that have very weak signal in the behavior of the final model.  The Enron negative sentiment property produces the weakest signal among all the experiments because (1) the feature does not exist in the dataset and (2) it has the smallest difference in percentage ($t_1-t_0$) among all the other experiments (5\% vs 10\%). We believe this is why it happens faster for the Sentiment property compared to other features. However, our insight suggests that this phenomenon should happen for any property for some poisoning ratio. To test this insight, we 

we tried various poisoning rates on the Enron dataset with random target feature. 
Figure \ref{fig:poisondrops} shows the result of this experiment where the accuracy of the attack starts to drop at poisoning rates around $30\%$. 
Interestingly, this phenomenon could be also explained using our Theorem as both ends of the range of certainty in the condition of our Theorem \ref{thm:theoritical_attack} will converge to $\infty$ when $p$ approaches $1$.
\begin{figure}
\begin{center}
    \includegraphics[scale=0.6]{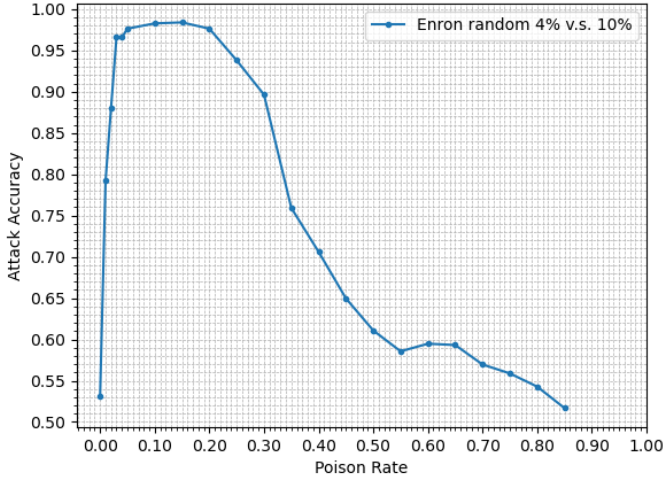}
    \caption{Poison rate vs attack accuracy. This experiment shows that more poisoning becomes ineffective after a certain point and actually decreases the performance of the attack. The optimal poisoning ratio could be different across different experiments. This drop of attack accuracy was also observed in one of the experiments of Figure \ref{fig:logisticpoison}. Note that adversary can choose to use fewer poisoning points to ensure that the attack is optimal. In this experiment, the number of shadow models is 400 and the training set size is 1000.}
    \label{fig:poisondrops}
\end{center}    
\end{figure}

\myparagraph{Number of Shadow Models} The next set of experiments (See Fig~\ref{fig:logisticshadow}) are to see the effect of the number of shadow models on the accuracy of the attack. For these experiments, we vary the number of shadow models from 50 to 2000. We notice that increasing the number of shadow models increases the attack accuracy and about 500 shadow models are sufficient to get close to maximum accuracy. Note that in this experiment we set the poisoning ratio to small values so that we can see the trend better. If larger poison ratio were chosen, in most experiments the attack reaches the maximum of 1 with very few shadow models and it is hard to see the trend. For instance, with $10\%$ percent poisoning, the experiments with random feature (both census and Enron) would reach $100\%$ accuracy with only 50 shadow models. This small number of shadow models makes the running time of the attack  lower. In all of the experiments in this figure, the dataset size is set to $1000$ except for the Enron negative sentiment experiment where the dataset size is $2000$.

\begin{figure}
    \begin{center}
    \includegraphics[scale=0.6]{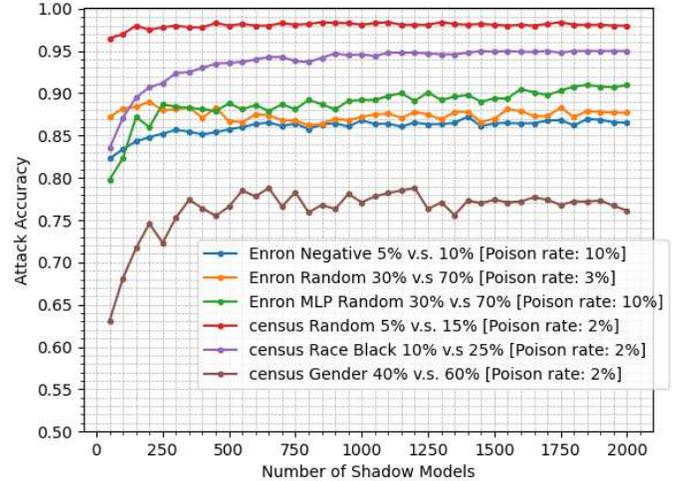}
    \caption{
    Number of shadow models v.s. attack accuracy. Our attack used shadow model training to find important queries. This figure shows that in all experiments, less than 500 shadow models is enough to get the maximum accuracy. Note that in this experiment we also  attack MLP classifier on Enron dataset for the random feature.}
    \label{fig:logisticshadow}
    \end{center}
\end{figure}

\myparagraph{Training Set size} In Fig.~\ref{fig:logisticsize}, we wanted to see the effect of training size on the effectiveness of the attack. Note that our theoretical attack suggests that larger training size should actually improve the attack because the models trained on larger training sets would have smaller generalization error and hence would be closer to a Bayes-optimal classifier. In Theorem \ref{thm:theoritical_attack} as the training set size increases, $\epsilon(n)$ and $\delta(n)$ will decrease which makes the attack more successful. In fact, our experiments verify this theoretical insight. In our experiments, we vary the training set size from 100 to 1500 and the upward trend is quite easy to observe.
 In this experiment we use 500 shadow models. Again, we have selected the poisoning rate and the number of shadow models in a way that the attack does not get accuracy 1.0 for small training sizes. In all experiments in this figure, the number of shadow models is set to 200. Also, the number of repetitions for this experiment is 2. 

\begin{figure}
    \begin{center}
    \includegraphics[scale=0.6]{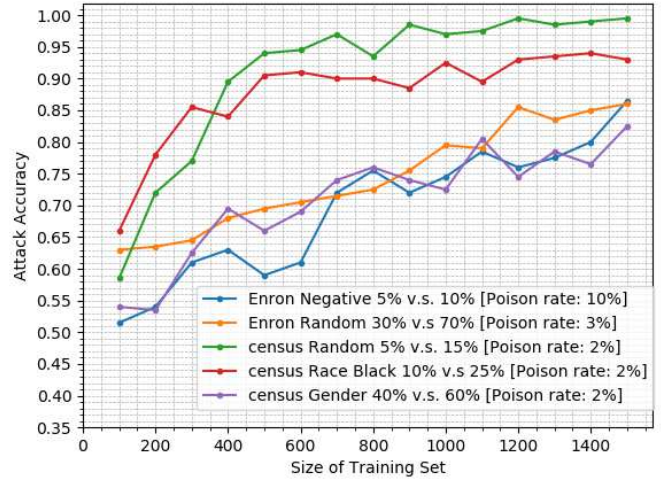}
    \caption{Training set size v.s. attack accuracy.
    As was observed in our theoretical analysis
    , the size of the training set can impact the performance of  our attack (because it affects the generalization error).}
    \label{fig:logisticsize}
    \end{center}
\end{figure}

\myparagraph{Undetectablilty of the Attack} Recall that in our threat model, the adversary is able to poison a fraction of the training data. If the target model quality degrades significantly due to this poisoning, then it becomes easily detectable. Therefore, for the effectiveness of this attack, it is important that the quality of the model does not degrade\footnote{Note that we are only considering undetectability via black box access. In a eyes-on setting where the training data can be looked at, it is possible to have countermeasures that would detect poisoning by looking at the training data}. We experimentally confirm that this is somewhat true with our poisoning strategy. See Fig~\ref{fig:logisticprecession} and Fig~\ref{fig:logisticrecall} for the precision and recall rate\footnote{Precision: $\frac{\text{True positive}}{\text{True positive} + \text{False positive}}$, Recall:$\frac{\text{True positive}}{\text{True positive} + \text{False negative}}$} 
 for the model Logistic Regression where the poisoning rate varies from 0\% to 20\% for training set size of 1000. In general, the experiments show that the precision tends to decrease  with a rather low slope and recall tend to increase by adding more poison data.
  The drop of precision and rise of recall can be explained by the fact that the poisoned model tends to predict label 1 more often than the non-poisoned model, because the poisoned data we add all has label 1.
 However, it also worth mentioning that for all experiments in Census data, 4-5\% poisoning is sufficient to get attack accuracy more than 90\%. This means that, if one considers the drop in precision versus the attack accuracy, the census data is not much worse than enron. In our experiments, the size of the training set for Census experiments is set to 1000 and for Enron experiments the training set size is 2000.

\begin{figure}
\begin{center}
    \includegraphics[scale=0.6]{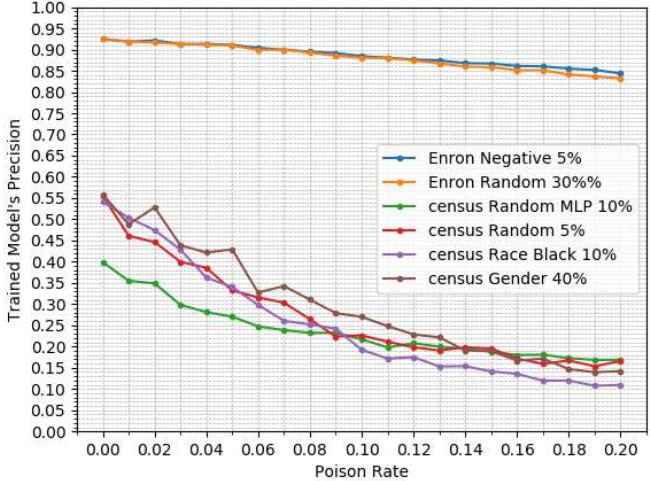}
    \caption{Precision vs poison rate. Note that we did not balance the labels in our experiments and this is why we have precision smaller than 50\% in the Census experiments.}
    \label{fig:logisticprecession}
    \end{center}
\end{figure}

\begin{figure}
\begin{center}
    \includegraphics[scale=0.6]{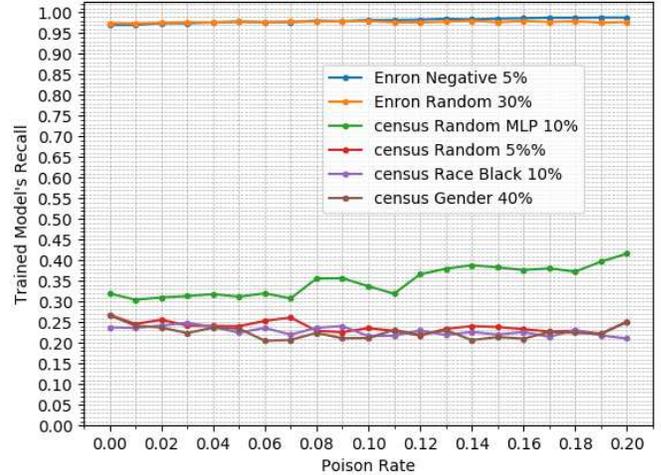}
    \caption{Recall v.s. poison rate. We change the poison rate and capture the recall of resulting models in different experiments. Note that we did not use balanced label fractions in our experiments.}
    \label{fig:logisticrecall}
\end{center}
\end{figure}

\myparagraph{Complexity of Target Models} While in most of our experiments we fix the target model to be logistic regression (except for few experiments named Census MLP and Enron MLP in the figures),
here we experiment with more complex architectures to see how our attack performs. We summarize the results in Table~\ref{tab:modelcomplexity}. Based on our theoretical analysis, the effectiveness of the attack should depend on the model's performance in generalization to the training data distribution. Therefore, we expect the effectiveness of the attack to drop with more complex networks as the generalization error would increase when the number of parameters in the model increases. This might sound counter intuitive as the privacy problems are usually tied with over fitting and unintended memorization\cite{secretsharer}. However, our experiments verify our theoretical intuition. We observe that as we add we more layers, the accuracy of the attack tends to drop. However, we would expect this to change with larger training set sizes as the larger training size could compensate for the generalization error caused by higher numbers of parameters and overfitting.  For instance, in the last row of Table~\ref{tab:modelcomplexity} the accuracy increases significantly when we set the training size to 10000 and use more shadow models.

\begin{table*}[htbp]
\begin{center}
\scalebox{1}{
\begin{tabular}{ | c | c | c | c | c | c |}
\hline
\multicolumn{3}{| c |}{\textbf{Architecture}}  & \multicolumn{2}{ c |}{\textbf{Performance}}\\ 
\cline{1-5}
\textbf{Hidden Layers}   & \textbf{Layer sizes} & \textbf{Training Size}  & \textbf{Attack Accuracy} & \textbf{Shadow Models}\\
\hline
 1 & [2] &  1000 & $1.0$ & 600\\\hline
 2 & [4 2] & 1000 &  $0.97$  & 600 \\\hline
 3 & [8 4 2] & 1000 & $0.94$  & 600 \\\hline
 4 & [16 8 4 2] & 1000 & $0.88$  & 600 \\\hline
 5 & [32 16 8 4 2] & 1000 & $0.81$  & 600 \\\hline
 5 & [32 16 8 4 2] & 10000 & $0.92$  & 1000 \\\hline 
\end{tabular}}
\caption{Complexity of target models vs attack accuracy on Census Data.}
\label{tab:modelcomplexity}
\end{center}
\end{table*}

\begin{table*}[!t]
\begin{center}
\scalebox{1}{
\begin{tabular}{| c | c | c | c | c | c | c | c | }
\hline
\multicolumn{2}{| c |}{\textbf{Experiment parameters}}  & \multicolumn{2}{ c |}{\textbf{White-box Performance \cite{PIA18}}} & \multicolumn{3}{ c |}{\textbf{Black-box Performance}}\\ 

 \textbf{Feature} &  \textbf{Training Size} & \textbf{\# Shadow Models}   & \textbf{Accuracy} & \textbf{\# Shadow Models } & \textbf{Poison} & \textbf{Accuracy}  \\
\hline
 Census-Random &  $1000$ & 1000 & $.52$ & $1000$ & $0$ &  .5\\ \hline
 Census-Gender & $1000$ & 1000  & $.96$ & 1000 & $0$ & .61\\ \hline
 Census-Race &  $1000$ & 1000 & $.95$ & 1000 &  $0$ & .55\\ \hline
 \hline
 Census-Random &  $1000$ & 1000 & $.52$ & $100$ & $0.05$ & $1.0$ \\ \hline
 Census-Gender & $1000$ & 1000  & $.96$ & 100 & $0.03$ & $.99$ \\ \hline
 Census-Race &  $1000$ & 1000 & $.95$ & 100 &  $0.05$ & $.97$\\ \hline
  \hline
  Census-Random &  $1000$ & 1000 & $.52$ & $50$ & $0.1$ & $1.0$ \\ \hline
 Census-Gender & $1000$ & 1000  & $.96$ & 50 & $0.1$ & $1.0$ \\ \hline
 Census-Race &  $1000$ & 1000 & $.95$ & 50 &  $0.1$ & $.98$\\ \hline
 
\end{tabular}}
\end{center}
\caption{Comparison with the white-box attack of \cite{PIA18}}
\label{tab:compareLR}
\end{table*}

\subsection{Comparison with $\wb$ \cite{PIA18}}

Since the work closest to ours is $\wb$, even though it is a white-box attack, we experimentally compare the performance of $\wb$ to ours.  $\wb$ is an improved version of property inference attack of \citep{propinf18}, where instead of a simple white-box shadow model training on the parameters of neural networks (which is done in \citep{ateniese2013hacking}), they first try to decrease the entropy of the model parameters by sorting the neural network neurons according to the size of their input and output. They show that this reduction in randomness of the shadow models can increase the accuracy of attack, for the same number of shadow models. This attack is called the \emph{vector} attack in \citep{PIA18}.  
In Table~\ref{tab:compareLR}, we see how our \emph{black-box attack performance} compares with $\wb$. Notice that black-box with no poisoning (first 3 rows of the table) performs much worse that $\wb$ on race and gender. However, $\wb$ performs poorly on the random feature. In fact, the strength of our attack is to find a way to infer information about features similar to random that do not have high correlation with the label (and this is exactly what $\wb$ is unable to attack), and do it in a fully black-box way. As we see in the columns below, with very small ratio of poisoning our attack get accuracy $1.0$ on the random target property. It also beats the performance of $\wb$ on other experiments with a very small number of poisoning points. Note that our attack also requires many fewer shadow models. 
For example with $10\%$ poisoning, only 50 shadow models in our attack would beat the accuracy of $\wb$ which uses $1000$ shadow models. The small number of shadow models would be important in scenarios where the adversary does not have access to a lot of similar data. So in summary, our attack can improve on the performance $\wb$ both in accuracy and number of shadow models, and of course in the access model which is black box.  The cost of these improvements is allowing the adversary to choose a fraction of training set which is not an uncommon scenario in multi-party learning applications. To compare the running time of the attacks, for the Census-Gender experiment in Table \ref{tab:compareLR}, $\wb$ ran in $1533$s compared to our black-box attack that only took $161$s on the same platform and execution environment. Note that we used the exact same size of held-out data for shadow models in both black-box and white-box experiments.

%%%% Stale tables
\begin{comment}
To compare performance on more complex networks, we run the following additional experiments on 3 layer  (Table~\ref{tab:compare3nn}) and observe the same trend as in the case of Logistic Regression.

%%%%%%%%%%%%%%%%%%%%%%%%%%%%%%%%%%%
\begin{table*}[htbp]
\begin{center}
\begin{tabular}{| c | c | c | c | c | c | c | }
\hline
\multicolumn{2}{| c |}{\textbf{Experiment parameters}}  & \multicolumn{2}{ c |}{\textbf{White-box Performance}} & \multicolumn{3}{ c |}{\textbf{Black-box Performance}}\\ 

 \textbf{Feature} &  \textbf{Training Size} & \textbf{\# Shadow Models}   & \textbf{Accuracy} & \textbf{\# Shadow Models } & \textbf{Poison} & \textbf{Accuracy}  \\
\hline
 Census-Random &  $1000$ & 1000 & $.54$ & $200$ & $0$ &  $0.5$\\ \hline
 Enron-"get" & $1000$ & 1000  & $.88$ & $200$ & $0$ & $0.74$\\ \hline
\hline
 Census-Random &  $1000$ & 1000 & $.54$ & $200$ & $0.1$ &  $0.89$\\ \hline
 Enron-"get" & $1000$ & 1000  & $.88$ & $200$ & $0.1$ & $0.92$\\ \hline
\end{tabular}
\end{center}
\caption{Comparison on Model: 3 layer MLP }
\label{tab:compareLR}
\end{table*}
\end{comment}
%%%%%%%%%%%%%%%%%%%%%%%%%%%%%%%

\section{Conclusion and Future Work}\label{sec:conclusion}
Poisoning attacks have emerged as a threat against security of machine learning. These poisoning attacks are usually studied in machine learning security where the goal of adversary is to increase the error or inject backdoor to the model. In this work, we initiated the study of poisoning adversaries that instead seek to the increase information leakage of trained models. We first presented a theoretical attack which succeeds against any Bayes optimal classifier. Then, we verified the effectiveness of out attack experimentally by implementing it and testing it against different datasets and classification tasks. Our results show that it does indeed succeed with very high probability. However, in our attack as well as in all the previous property attacks in the literature~\cite{PIA15, PIA18}, the attacker is assumed to have sample access to the underlying distributions $D_{-}$ and $D_{+}$. It would be interesting to remove/relax this assumption in future work.

 In this work, we tried to take a principled approach where experiments demonstrate the realistic threat of our attack, whereas the theoretical analysis gives intuition for the attacks. The theoretical analysis will hopefully also provide suggestion for where we might look in the future to either strengthen the attack or find mitigation. We leave this direction as future work.

\newpage
\begingroup
\raggedright
\emergencystretch 1.5em
\printbibliography

@article{PIA15,
  author    = {Giuseppe Ateniese and
               Luigi V. Mancini and
               Angelo Spognardi and
               Antonio Villani and
               Domenico Vitali and
               Giovanni Felici},
  title     = {Hacking smart machines with smarter ones: How to extract meaningful
               data from machine learning classifiers},
  journal   = {{IJSN}},
  volume    = {10},
  number    = {3},
  pages     = {137--150},
  year      = {2015},
  url       = {https://doi.org/10.1504/IJSN.2015.071829},
  doi       = {10.1504/IJSN.2015.071829},
  bibsource = {dblp computer science bibliography, https://dblp.org}
}

@inproceedings{DMNS06,
  author    = {Cynthia Dwork and
               Frank McSherry and
               Kobbi Nissim and
               Adam D. Smith},
  editor    = {Shai Halevi and
               Tal Rabin},
  title     = {Calibrating Noise to Sensitivity in Private Data Analysis},
  booktitle = {Theory of Cryptography, Third Theory of Cryptography Conference, {TCC}
               2006, New York, NY, USA, March 4-7, 2006, Proceedings},
  series    = {Lecture Notes in Computer Science},
  volume    = {3876},
  pages     = {265--284},
  publisher = {Springer},
  year      = {2006},
  url       = {https://doi.org/10.1007/11681878\_14},
  doi       = {10.1007/11681878\_14},
  bibsource = {dblp computer science bibliography, https://dblp.org}
}

@misc{he2019privacy,
    title={Towards Privacy and Security of Deep Learning Systems: A Survey},
    author={Yingzhe He and Guozhu Meng and Kai Chen and Xingbo Hu and Jinwen He},
    year={2019},
    eprint={1911.12562},
    archivePrefix={arXiv},
    primaryClass={cs.CR}
}

@inproceedings{propinf18,
author = {Ganju, Karan and Wang, Qi and Yang, Wei and Gunter, Carl A. and Borisov, Nikita},
title = {Property Inference Attacks on Fully Connected Neural Networks Using Permutation Invariant Representations},
year = {2018},
isbn = {9781450356930},
publisher = {Association for Computing Machinery},
address = {New York, NY, USA},
url = {https://doi.org/10.1145/3243734.3243834},
doi = {10.1145/3243734.3243834},
booktitle = {Proceedings of the 2018 ACM SIGSAC Conference on Computer and Communications Security},
pages = {619–633},
numpages = {15},
location = {Toronto, Canada},
series = {CCS ’18}
}

@article{melis19,
   title={Exploiting Unintended Feature Leakage in Collaborative Learning},
   ISBN={9781538666609},
   url={http://dx.doi.org/10.1109/SP.2019.00029},
   DOI={10.1109/sp.2019.00029},
   journal={2019 IEEE Symposium on Security and Privacy (SP)},
   publisher={IEEE},
   author={Melis, Luca and Song, Congzheng and De Cristofaro, Emiliano and Shmatikov, Vitaly},
   year={2019},
   month={May}
}

@misc{chen17,
    title={Targeted Backdoor Attacks on Deep Learning Systems Using Data Poisoning},
    author={Xinyun Chen and Chang Liu and Bo Li and Kimberly Lu and Dawn Song},
    year={2017},
    eprint={1712.05526},
    archivePrefix={arXiv},
    primaryClass={cs.CR}
}

@article{optimalGB,
  title={White-box vs Black-box: Bayes Optimal Strategies for Membership Inference},
  author={Alexandre Sablayrolles and Matthijs Douze and Cordelia Schmid and Yann Ollivier and Herv{\'e} J{\'e}gou},
  journal={ArXiv},
  year={2019},
  volume={abs/1908.11229}
}

@inproceedings{SubpopulationDP,
  title={Subpopulation Data Poisoning Attacks},
  author={Matthew Jagielski},
  year={2019}
}

@inproceedings{PIA18,
  title={Property Inference Attacks on Fully Connected Neural Networks using Permutation Invariant Representations},
  author={Karan Ganju and Qi Wang and Wei Yang and Carl A. Gunter and Nikita Borisov},
  booktitle={CCS '18},
  year={2018}
}

@article{ateniese2013hacking,
  title={Hacking smart machines with smarter ones: How to extract meaningful data from machine learning classifiers},
  author={Ateniese, Giuseppe and Felici, Giovanni and Mancini, Luigi V and Spognardi, Angelo and Villani, Antonio and Vitali, Domenico},
  journal={arXiv preprint arXiv:1306.4447},
  year={2013}
}

@misc{ericsson,
  title = {3 ways to train a secure machine learning model},
  howpublished = {https://www.ericsson.com/en/blog/2020/2/training-a-machine-learning-model},
  note = {Accessed: 2020-03-04}
}

@misc{medium,
  title = {Privacy-preserving Collaborative Machine Learning},
  howpublished = {https://medium.com/sap-machine-learning-research/privacy-preserving-collaborative-machine-learning-35236870cd43},
  note = {Accessed: 2020-03-04}
}

@inproceedings{secretsharer,
  title={The Secret Sharer: Evaluating and Testing Unintended Memorization in Neural Networks},
  author={Nicholas Carlini and Chang Liu and {\'U}lfar Erlingsson and Jernej Kos and Dawn Xiaodong Song},
  booktitle={USENIX Security Symposium},
  year={2018}
}

@article{frank2011uci,
  title={UCI machine learning repository, 2010},
  author={Frank, Andrew and Asuncion, Arthur and others},
  journal={URL http://archive. ics. uci. edu/ml},
  volume={15},
  pages={22},
  year={2011}
}

@artcile{klimt2004enron,
  title={The enron corpus: A new dataset for email classification research},
  author={Klimt, Bryan and Yang, Yiming},
  journal={European Conference on Machine Learning},
  pages={217--226},
  year={2004}
}

@inproceedings{shokri2017membership,
  title={Membership inference attacks against machine learning models},
  author={Shokri, Reza and Stronati, Marco and Song, Congzheng and Shmatikov, Vitaly},
  booktitle={2017 IEEE Symposium on Security and Privacy (SP)},
  pages={3--18},
  year={2017},
  organization={IEEE}
}

@article{biggio2012poisoning,
  title={Poisoning attacks against support vector machines},
  author={Biggio, Battista and Nelson, Blaine and Laskov, Pavel},
  journal={arXiv preprint arXiv:1206.6389},
  year={2012}
}

@inproceedings{steinhardt2017certified,
  title={Certified defenses for data poisoning attacks},
  author={Steinhardt, Jacob and Koh, Pang Wei W and Liang, Percy S},
  booktitle={Advances in neural information processing systems},
  pages={3517--3529},
  year={2017}
}

@inproceedings{shafahi2018poison,
  title={Poison frogs! targeted clean-label poisoning attacks on neural networks},
  author={Shafahi, Ali and Huang, W Ronny and Najibi, Mahyar and Suciu, Octavian and Studer, Christoph and Dumitras, Tudor and Goldstein, Tom},
  booktitle={Advances in Neural Information Processing Systems},
  pages={6103--6113},
  year={2018}
}

@inproceedings{mahloujifar2019curse,
  title={The curse of concentration in robust learning: Evasion and poisoning attacks from concentration of measure},
  author={Mahloujifar, Saeed and Diochnos, Dimitrios I and Mahmoody, Mohammad},
  booktitle={Proceedings of the AAAI Conference on Artificial Intelligence},
  volume={33},
  pages={4536--4543},
  year={2019}
}

@inproceedings{nasr2019comprehensive,
  title={Comprehensive privacy analysis of deep learning},
  author={Nasr, Milad and Shokri, Reza and Houmansadr, Amir},
  booktitle={2019 ieee symposium on security and privacy},
  year={2019}
}

@inproceedings{nasr2018machine,
  title={Machine learning with membership privacy using adversarial regularization},
  author={Nasr, Milad and Shokri, Reza and Houmansadr, Amir},
  booktitle={Proceedings of the 2018 ACM SIGSAC Conference on Computer and Communications Security},
  pages={634--646},
  year={2018}
}

@article{choo2020label,
  title={Label-Only Membership Inference Attacks},
  author={Choo, Christopher A Choquette and Tramer, Florian and Carlini, Nicholas and Papernot, Nicolas},
  journal={arXiv preprint arXiv:2007.14321},
  year={2020}
}

@inproceedings{wang2018stealing,
  title={Stealing hyperparameters in machine learning},
  author={Wang, Binghui and Gong, Neil Zhenqiang},
  booktitle={2018 IEEE Symposium on Security and Privacy (SP)},
  pages={36--52},
  year={2018},
  organization={IEEE}
}

@article{song2020systematic,
  title={Systematic Evaluation of Privacy Risks of Machine Learning Models},
  author={Song, Liwei and Mittal, Prateek},
  journal={arXiv preprint arXiv:2003.10595},
  year={2020}
}

@inproceedings{jayaraman2019evaluating,
  title={Evaluating differentially private machine learning in practice},
  author={Jayaraman, Bargav and Evans, David},
  booktitle={28th $\{$USENIX$\}$ Security Symposium ($\{$USENIX$\}$ Security 19)},
  pages={1895--1912},
  year={2019}
}

@article{jagielski2020auditing,
  title={Auditing Differentially Private Machine Learning: How Private is Private SGD?},
  author={Jagielski, Matthew and Ullman, Jonathan and Oprea, Alina},
  journal={Advances in Neural Information Processing Systems},
  volume={33},
  year={2020}
}

@inproceedings{fredrikson2015model,
  title={Model inversion attacks that exploit confidence information and basic countermeasures},
  author={Fredrikson, Matt and Jha, Somesh and Ristenpart, Thomas},
  booktitle={Proceedings of the 22nd ACM SIGSAC Conference on Computer and Communications Security},
  pages={1322--1333},
  year={2015}
}

@inproceedings{song2019auditing,
  title={Auditing data provenance in text-generation models},
  author={Song, Congzheng and Shmatikov, Vitaly},
  booktitle={Proceedings of the 25th ACM SIGKDD International Conference on Knowledge Discovery \& Data Mining},
  pages={196--206},
  year={2019}
}

@article{goldblum2020data,
  title={Data Security for Machine Learning: Data Poisoning, Backdoor Attacks, and Defenses},
  author={Goldblum, Micah and Tsipras, Dimitris and Xie, Chulin and Chen, Xinyun and Schwarzschild, Avi and Song, Dawn and Madry, Aleksander and Li, Bo and Goldstein, Tom},
  journal={arXiv preprint arXiv:2012.10544},
  year={2020}
}

@article{suya2020model,
  title={Model-Targeted Poisoning Attacks: Provable Convergence and Certified Bounds},
  author={Suya, Fnu and Mahloujifar, Saeed and Evans, David and Tian, Yuan},
  journal={arXiv preprint arXiv:2006.16469},
  year={2020}
}

@article{mahloujifar2020learning,
  title={Learning under p-tampering poisoning attacks},
  author={Mahloujifar, Saeed and Diochnos, Dimitrios I and Mahmoody, Mohammad},
  journal={Annals of Mathematics and Artificial Intelligence},
  volume={88},
  number={7},
  pages={759--792},
  year={2020},
  publisher={Springer}
}

@article{garg2020obliviousness,
  title={Obliviousness Makes Poisoning Adversaries Weaker},
  author={Garg, Sanjam and Jha, Somesh and Mahloujifar, Saeed and Mahmoody, Mohammad and Thakurta, Abhradeep},
  journal={arXiv preprint arXiv:2003.12020},
  year={2020}
}

@article{diochnos2019lower,
  title={Lower bounds for adversarially robust pac learning},
  author={Diochnos, Dimitrios I and Mahloujifar, Saeed and Mahmoody, Mohammad},
  journal={arXiv preprint arXiv:1906.05815},
  year={2019}
}

@inproceedings{mahloujifar2019can,
  title={Can Adversarially Robust Learning Leverage Computational Hardness?},
  author={Mahloujifar, Saeed and Mahmoody, Mohammad},
  booktitle={Algorithmic Learning Theory},
  pages={581--609},
  year={2019},
  organization={PMLR}
}

@inproceedings{mahloujifar2019data,
  title={Data poisoning attacks in multi-party learning},
  author={Mahloujifar, Saeed and Mahmoody, Mohammad and Mohammed, Ameer},
  booktitle={International Conference on Machine Learning},
  pages={4274--4283},
  year={2019},
  organization={PMLR}
}

@inproceedings{mahloujifar2018learning,
  title={Learning under $ p $-Tampering Attacks},
  author={Mahloujifar, Saeed and Diochnos, Dimitrios I and Mahmoody, Mohammad},
  booktitle={Algorithmic Learning Theory},
  pages={572--596},
  year={2018},
  organization={PMLR}
}

@inproceedings{mahloujifar2017blockwise,
  title={Blockwise p-tampering attacks on cryptographic primitives, extractors, and learners},
  author={Mahloujifar, Saeed and Mahmoody, Mohammad},
  booktitle={Theory of Cryptography Conference},
  pages={245--279},
  year={2017},
  organization={Springer}
}

@inproceedings{bhagoji2019analyzing,
  title={Analyzing federated learning through an adversarial lens},
  author={Bhagoji, Arjun Nitin and Chakraborty, Supriyo and Mittal, Prateek and Calo, Seraphin},
  booktitle={International Conference on Machine Learning},
  pages={634--643},
  year={2019},
  organization={PMLR}
}

@inproceedings{bagdasaryan2020backdoor,
  title={How to backdoor federated learning},
  author={Bagdasaryan, Eugene and Veit, Andreas and Hua, Yiqing and Estrin, Deborah and Shmatikov, Vitaly},
  booktitle={International Conference on Artificial Intelligence and Statistics},
  pages={2938--2948},
  year={2020},
  organization={PMLR}
}

@article{goodfellow2014explaining,
  title={Explaining and harnessing adversarial examples},
  author={Goodfellow, Ian J and Shlens, Jonathon and Szegedy, Christian},
  journal={arXiv preprint arXiv:1412.6572},
  year={2014}
}

@article{ilyas2019adversarial,
  title={Adversarial examples are not bugs, they are features},
  author={Ilyas, Andrew and Santurkar, Shibani and Tsipras, Dimitris and Engstrom, Logan and Tran, Brandon and Madry, Aleksander},
  journal={arXiv preprint arXiv:1905.02175},
  year={2019}
}

@article{madry2017towards,
  title={Towards deep learning models resistant to adversarial attacks},
  author={Madry, Aleksander and Makelov, Aleksandar and Schmidt, Ludwig and Tsipras, Dimitris and Vladu, Adrian},
  journal={arXiv preprint arXiv:1706.06083},
  year={2017}
}

@inproceedings{song2019privacy,
  title={Privacy risks of securing machine learning models against adversarial examples},
  author={Song, Liwei and Shokri, Reza and Mittal, Prateek},
  booktitle={Proceedings of the 2019 ACM SIGSAC Conference on Computer and Communications Security},
  pages={241--257},
  year={2019}
}

@article{koh2018stronger,
  title={Stronger data poisoning attacks break data sanitization defenses},
  author={Koh, Pang Wei and Steinhardt, Jacob and Liang, Percy},
  journal={arXiv preprint arXiv:1811.00741},
  year={2018}
}

@article{diakonikolas2019robust,
  title={Robust estimators in high-dimensions without the computational intractability},
  author={Diakonikolas, Ilias and Kamath, Gautam and Kane, Daniel and Li, Jerry and Moitra, Ankur and Stewart, Alistair},
  journal={SIAM Journal on Computing},
  volume={48},
  number={2},
  pages={742--864},
  year={2019},
  publisher={SIAM}
}

@inproceedings{diakonikolas2019sever,
  title={Sever: A robust meta-algorithm for stochastic optimization},
  author={Diakonikolas, Ilias and Kamath, Gautam and Kane, Daniel and Li, Jerry and Steinhardt, Jacob and Stewart, Alistair},
  booktitle={International Conference on Machine Learning},
  pages={1596--1606},
  year={2019},
  organization={PMLR}
}

@article{prasad2018robust,
  title={Robust estimation via robust gradient estimation},
  author={Prasad, Adarsh and Suggala, Arun Sai and Balakrishnan, Sivaraman and Ravikumar, Pradeep},
  journal={arXiv preprint arXiv:1802.06485},
  year={2018}
}

@inproceedings{lai2016agnostic,
  title={Agnostic estimation of mean and covariance},
  author={Lai, Kevin A and Rao, Anup B and Vempala, Santosh},
  booktitle={2016 IEEE 57th Annual Symposium on Foundations of Computer Science (FOCS)},
  pages={665--674},
  year={2016},
  organization={IEEE}
}
\endgroup
\clearpage

\appendix
\section{Omitted Proofs}
In this section we prove the tools we used to prove our main Theorem. Specifically, We prove Claim \ref{clm1}, Corollary \ref{cor1} and Claim \ref{clm2}.
\begin{proof}[Proof of Claim \ref{clm1}] We have
    \begin{align}\nonumber
        &\Pr[\tY=1 \mid \tX = x]\\\nonumber
        &~~~= \Pr[\tY=1 \mid \tX = x \ANDT E]\cdot \Pr[E \mid \tX = x ] \\\nonumber
        &~~~~~~+ \Pr[\tY=1 \mid \tX = x \ANDT \bar{E}]\cdot \Pr[\bar{E} \mid \tX = x ]\\%\nonumber
       &~~~= \Pr[E \mid \tX = x ] 
    %   \\\nonumber &~~~~~~
        + \Pr[Y=1 \mid X = x]\cdot \Pr[\bar{E} \mid \tX = x ]\label{eq-3}
    \end{align} 

 Now we should calculate the probability $\Pr[E  \mid \tX = x]$ to get the exact probabilities. We have
    \begin{align}
\nonumber &\Pr[E \mid \tX = x]\\\nonumber
&~~~= \frac{\Pr[\tX = x \mid E]\cdot \Pr[E]}{\Pr[\tX = x \mid E]\cdot \Pr[E] +\Pr[\tX = x \mid \Bar{E}]\cdot \Pr[\Bar{E}]}\\
&~~~= \frac{\Pr[X_+=x]\cdot p}{\Pr[X_+ = x]\cdot p + \Pr[X=x]\cdot(1-p)}\label{eq0}
\end{align}
% On the other hand for all $x\in \cX$ we have
% \begin{equation}\label{eq1}\Pr[X_A = x] = \Pr[X_{+} = x].\end{equation} 
% and for all $x\in\cX$ such that $f(x)=0$ we have 

% \begin{equation}\label{eq2}\Pr[X_A = x] = (1-m_A)\cdot \Pr[X_{-} = x].\end{equation}

Now for all $x\in\cX$ such that $f(x)=1$ we have 
\begin{equation}\label{eq3}\Pr[X = x] = t\cdot \Pr[X_{+} = x].\end{equation}
and for all $x\in\cX$ such that $f(x)=0$ we have 

% \begin{equation}\label{eq4}\Pr[X = x] = (1-t)\cdot \Pr[X_{-} = x].\end{equation}

Now combining Equations \ref{eq0} and \ref{eq3}, for all $x\in\cX$ such that $f(x)=1$ we have 
\begin{equation}\label{eq6}\Pr[E \mid \tX = x]=\frac{ p}{ p + t\cdot (1-p)}.\end{equation}

Combining equations \ref{eq-3} and \ref{eq6} 
we get \esha{do we need \ref{eq4} for this?}

\begin{align*}\Pr[\tilde{Y}=1 \mid \tilde{X}=x] &= \frac{p}{p+t(1-p)}\\
&~~~+ \frac{t(1-p)}{p+t(1-p)}\cdot \Pr[Y=1\mid X=x]
\end{align*}

which finishes the proof.
\end{proof}

\begin{proof}[Proof of Corollary \ref{cor1}]
if $\crt(x)\leq  \frac{p-2\tau\cdot t}{t(1-p)}$ 
then by Claim \ref{clm1} we have
\begin{align*}&\Pr[\tilde{Y}=1 \mid \tilde{X}=x]\\ &= \frac{p}{p+t(1-p)}+ \frac{t(1-p)}{p+t(1-p)}\cdot \Pr[Y=1\mid X=x]\\
&=\frac{p}{p+t(1-p)} +  \frac{t(1-p)}{p+t(1-p)}\cdot (\frac{1-\crt(x)}{2})\\
&\geq \frac{p}{p+t(1-p)}
+ \frac{t(1-p)}{p+t(1-p)}\cdot (\frac{t(1-p) -p +2t\tau }{2t(1-p)})\\ 
&=\frac{t(1-p)+p +2\tau t}{2(p+t(1-p))}\\
&= \frac{1}{2} +\frac{\tau t}{p+t(1-p)}.
\end{align*}
To show the other direction, we can follow the exact same steps in  the opposite order.
\end{proof}

\begin{proof}[Proof of Claim \ref{clm2}]
For all $x\in \cX$ such that $C_\tau(x)=1$, using Corollary \ref{cor1}, if $t= t_1$ then we have
\begin{equation}\label{eq20}\Pr[\tilde{Y}=1\mid \tX=x]< 0.5-\frac{\tau\cdot t_1}{p + (1-p)\cdot t_1}\end{equation}

and if $t=t_0$ then 
\begin{equation}\label{eq21}\Pr[\tY=1\mid \tX=x]\geq 0.5 +\frac{\tau\cdot t_0}{p + (1-p)\cdot t_0}\end{equation} %\melissa{$\tY,\tX$?} %\melissa{How 
This implies that for both cases of $t=t_0$ and $t=t_1$ we have 
\begin{equation}\label{eq27}|\crt(x,\tD)|\geq \frac{2\tau t}{p + (1-p)\cdot t}\end{equation}
And it also implies that for case of $t=t_0$ we have %\melissa{for these we need $\tau\geq 0$, right?}

\begin{equation}\label{eq28}h^*(x,\tD)=1 \end{equation}
and for $t=t_1$ we have
\begin{equation}\label{eq29}h^*(x,\tD)=0.\end{equation}

Now we state a lemma that will be used in the rest of the proof:

\begin{lemma}\label{lem2}
For any distribution $D\equiv (X,Y)$ where $\Supp(Y)=\set{0,1}$ and any classifier $h:\Supp(X)\to \set{0,1}$ we have:
$$\Risk(h,D) = \Bayes(D) + \Ex_{x\gets X}\left[|\tilde{h}x)-h^*(x,D)|\cdot |\crt(x,D)|\right]$$
where $h^*$ is the Bayes-Optimal classifier as defined in Definition \ref{def:bayes}.
\end{lemma}
\begin{proof} For simplicity, in this proof we use $h^*(x)$ instead of $h^*(x,D)$. We have
\begin{align*}
    &\Risk(h,D)\\
    &=\Ex_{(x,y)\gets(X,Y)}[\tilde{h}(x)\neq y]\\
    &=\Ex_{(x,y)\gets(X,Y)}[h^*(x)\neq y \mid \tilde{h}(x)=h^*(x)]\Pr[\tilde{h}(X)= h^*(X)]\\
    &+\Ex_{(x,y)\gets(X,Y)}[h^*(x)= y \mid \tilde{h}(x)\neq h^*(x)]\Pr[\tilde{h}(X)\neq h^*(X)]\\
    &=\Bayes(D)\\
    &- \Ex_{(x,y)\gets(X,Y)}[h^*(x)\neq  y \mid \tilde{h}(x)\neq h^*(x)]\Pr[\tilde{h}(X)\neq h^*(X)]\\
    &+\Ex_{(x,y)\gets(X,Y)}[h^*(x)= y \mid \tilde{h}(x)\neq h^*(x)]\Pr[\tilde{h}(X)\neq h^*(X)]\\
    &=\Bayes(D)\\
    &- \Ex_{(x,y)\gets(X,Y)}[|y-h^*(x)||\tilde{h}(x)-h^*(x)|]\\
    &+ \Ex_{(x,y)\gets(X,Y)}[(1-|y-h^*(x)|)|\tilde{h}(x)-h^*(x)|]\\
    &=\Bayes(D) + \Ex_{(x,y)\gets(X,Y)}[(1-2|y-h^*(x)|)|\tilde{h}(x)-h^*(x)|]\\
    &=\Bayes(D) + \Ex_{x\gets X}[|\tilde{h}(x)-h^*(x)|\Ex_{y\gets Y|X=x}[1-2|y-h^*(x)|]]
\end{align*}
Now we show that $\Ex_{y\gets Y|X=x}[1-2|y-h^*(x)|]=|\crt(x)|$. The reason is that, if $\Pr[Y=1|X=x]\geq 0.5$ then $h^*(x)=1$ and we have $\Ex_{y\gets Y|X=x}[1-2|y-h^*(x)|]=2\Ex_{y\gets Y|X=x}[y]-1 = |\crt(x)|$. And if $\Pr[Y=1|X=x]< 0.5$ then $h^*(x)=0$ and we have $\Ex_{y\gets Y|X=x}[1-2|y-h^*(x)|]=1-2\Ex_{y\gets Y|X=x}[y] = |\crt(x)|$. %\melissa{math formatting is messed up here I think?  It should be "$=1-2\Ex_{y\gets Y|X=x}[y] = |\crt(x)|$?}\Snote{Yes, fixed it. Thanks.} 
Therefore, the proof is complete.
\end{proof}
Now we are ready to complete the proof. Assume that we have $t=t_0$ and
\begin{equation}\label{eq22}
    \Pr_{x\gets X|C_\tau(x)=1}[\tilde{h}(x)=1] < 0.5+\gamma
\end{equation} 
let $\alpha=\Risk(\tilde{h},D)-\Bayes(D)$. \melissa{$\tilde{h}$?  Also below?}By Lemma \ref{lem2} we have

\begin{align}
    &\alpha = \Ex_{x\gets \tX}\left[|\tilde{h}(x)-h^*(x,\tD)|\cdot |\crt(x,\tD)|\right]\nonumber\\
    &\geq \Ex_{x\gets \tX}\left[|\tilde{h}x)-h^*(x,\tD)|\cdot |\crt(x,\tD)| \mid C_\tau(x) = 1\right]\cdot \Pr[C_\tau(\tX)=1]\nonumber\\
    % &\geq \Pr_{(x,y)\gets (\tX,\tY)\mid C_\tau(\tX)=1}[\tilde{h}(x)\neq y]\cdot \Pr[C_\tau(\tX)=1]\nonumber\\
    % &= (1-\Pr_{(x,y)\gets (\tX,\tY)\mid C_\tau(\tX)=1}[\tilde{h}(x)=y])\cdot \Pr[C_\tau(\tX)])\\&~~~~~ \mbox{ \melissa{is $C$ in the$ |C_\tau(\tX)$ the same as $C_\tau$?}}\\    
    % &= (1-\Pr_{(x,y)\gets (\tX,\tY)\mid C_\tau(\tX)=1}[\tilde{h}(x)=0\ANDT y=0] -\Pr_{(x,y)\gets (\tX,\tY)\mid C_\tau(\tX)=1}[\tilde{h}(x)=1\ANDT y=1] )\cdot \Pr[C_\tau(\tX)])\\&~~~~~ \mbox{ \melissa{is $C$ in the$ |C_\tau(\tX)$ the same as $C_\tau$?}}\\
    % &\geq (1-\Pr_{(x,y)\gets (\tX,\tY)\mid C_\tau(\tX)=1}[\tilde{h}(x)=0] -\Pr_{(x,y)\gets (\tX,\tY)\mid C_\tau(\tX)=1}[\tilde{h}(x)=1\ANDT y=1] )\cdot \Pr[C_\tau(\tX)])\\&~~~~~ \mbox{ \melissa{is $C$ in the$ |C_\tau(\tX)$ the same as $C_\tau$?}}\\
    &\geq (\Ex_{x\gets \tX\mid C_\tau(\tX)=1}[1-\tilde{h}(x)])\cdot\frac{2\tau t_0}{p + (1-p)\cdot t_0}\cdot\Pr[C_\tau(\tX)=1].\label{eq39}
    \end{align}
    Note that the last line above is derived by combining Equations \ref{eq27} and \ref{eq28}.
   
    Now using our assumption in Equation \ref{eq22} we have
    
    \begin{align*}
    \alpha>  \frac{2(0.5-\gamma)\tau\cdot t_0}{(p+(1-p)t_0)}\cdot \Pr[C_\tau(\tX)=1]
    \end{align*}
    Now we note that for any $x$ such that $f(x)=1$ we have  $$\Pr[\tX=x] = \frac{p + (1-p)t}{t}\cdot \Pr[X=x])$$ because of the way poisoning is done
    \footnote{For simplicity we are assuming the support of distribution is discrete. Otherwise, we have to have the same argument for measurable subsets instead of individual instances.}. Therefore we have 
    \begin{equation}\label{eq40}\Pr[C_\tau(\tX)=1]=\frac{p + (1-p)t}{t}\cdot \Pr[C_\tau(X)=1])\end{equation}
    
    and we get
    \begin{align*}
      \alpha &> \frac{2(0.5-\gamma)\tau\cdot t_0}{(p+(1-p)t_0)}\cdot\frac{p + (1-p)t_0}{t_0}\cdot \Pr[C_\tau(X)=1])\\
      &=(1-2\gamma)\cdot\tau\cdot \Pr[C_\tau(X)=1])\\
      &\geq (1-2\gamma)\cdot\tau\frac{\epsilon(n)}{\tau(1-2\gamma)}=\epsilon(n).
\end{align*}
This means that if the assumption of Equation \ref{eq22} holds then the error would be larger than $\epsilon(n) + \Bayes(\tD)$ which means the probability of \ref{eq22} happening is at most $\delta(n)$  by Bayes-optimality of the learning algorithm. Namely,
\begin{align*}
             \Pr_{\substack{S\gets \tilde{D}^n\\
             \tilde{h}\gets L(S)}}\left[\Pr_{x\gets X|C_\tau(x)=1}\left[\tilde{h}(x) =1\right] \geq0.5 +\gamma\right]\geq 1-\delta(n).
   \end{align*}
In case of $t=t_1$, the proof is similar. We first assume 
\begin{equation}\label{eq30}
    \Pr_{x\gets X|C_\tau(x)=1}[\tilde{h}(x)=0] > 0.5 -\gamma
\end{equation}
Using Equations \ref{eq27} and \ref{eq29} we get the following (similar to Inequality \ref{eq39} for $t=t_0$)
\begin{equation}\alpha>(\Ex_{x\gets \tX\mid C_\tau(\tX)=1}[\tilde{h}(x)])\cdot\frac{2\tau t_1}{p + (1-p)\cdot t_1}\cdot\Pr[C_\tau(\tX)=1].\label{eq31}
  \end{equation}
Therefore, using Equations \ref{eq40} and \ref{eq30} we get
$$\alpha> \epsilon(n) $$
which implies
  \begin{align*}
             \Pr_{\substack{S\gets \tilde{D}^n\\
             \tilde{h}\gets L(S)}}\left[\Pr_{x\gets X|C_\tau(x)=1}\left[\tilde{h}(x) =1\right] \leq 0.5 -\gamma\right]\geq 1-\delta(n).
   \end{align*}
\end{proof}

\typeout{get arXiv to do 4 passes: Label(s) may have changed. Rerun}
\end{document}